\documentclass[letterpaper, 10 pt, conference]{ieeeconf}

\IEEEoverridecommandlockouts
\providecommand{\citep}[1]{\cite{#1}}
\providecommand{\citet}[1]{\cite{#1}}

\usepackage[utf8]{inputenc} 
\usepackage[T1]{fontenc}    
\usepackage{hyperref}       
\usepackage{url}            
\usepackage{booktabs}       
\usepackage{amsfonts}       
\usepackage{nicefrac}       
\usepackage{microtype}      
\usepackage{xcolor}         
\usepackage{algorithm}      
\usepackage{algorithmic}
\usepackage{amssymb}        
\usepackage{diagbox}      
\usepackage[normalem]{ulem} 
\usepackage{bm}

\usepackage{amsmath}

\usepackage{amsthm}
\usepackage{empheq}

\usepackage{mathtools}

\usepackage{wrapfig,soul}
\usepackage{xcolor}
\usepackage{grffile}
\usepackage{thm-restate}
\usepackage{lipsum}
\usepackage{titlesec}
\usepackage{caption} 
\usepackage{subcaption} 
\usepackage{multirow} 
\usepackage{makecell} 
\usepackage{graphics}
\usepackage{graphbox}       
\usepackage{pifont}

\newtheorem{Definition}{Definition}

\newtheorem{Proposition}{Proposition}
\newtheorem{Theorem}{Theorem}
\newtheorem{Assumption}{Assumption}
\newtheorem{Remark}{Remark}

\newcommand{\R}{\mathbb{R}}

\title{\LARGE \bf
Balancing Multi-modal Sensor Learning via Multi-objective Optimization
}

\author{Heshan Fernando, Quan Xiao, Parikshit Ram, Yi Zhou, Horst Samulowitz, Nathalie Baracaldo, and Tianyi Chen%
\thanks{This work was supported by IBM through the IBM-Rensselaer Future of Computing Research Collaboration. Heshan Fernando and Tianyi Chen are with Rensselaer Polytechnic Institute, Troy, NY, USA. Parikshit Ram, Yi Zhou, Horst Samulowitz, and Nathalie Baracaldo are with IBM Research, USA. 
Quan Xiao and Tianyi Chen are with Cornell University, Ithaca, NY, USA.}%
}

\begin{document}

\maketitle
\thispagestyle{empty}
\pagestyle{empty}

\begin{abstract}
Learning-enabled control systems increasingly rely on multiple sensing modalities (e.g., vision, audio, language, etc.) for perception and decision support. A key challenge is that multi-modal sensor training dynamics are often imbalanced: fast-to-learn sensing channels dominate optimization, while slower channels remain underutilized, degrading reliability under sensing perturbations. Existing balancing strategies are largely heuristic and can require computationally intensive subroutines. In this paper, we reformulate multi-modal sensor learning as a multi-objective optimization (MOO) problem that explicitly prioritizes the worst-performing modality while retaining the nominal multi-modal sensor fusion objective. We then propose a simple gradient-based method, MIMO (multi-modal sensor learning via MOO), for the resulting formulation. We provide convergence guarantees and evaluate the method on standard multi-modal benchmarks. Results show improved balanced performance over state-of-the-art balanced multi-modal learning and MOO baselines, together with up to $\sim20\times$ reduction in subroutine computation time, highlighting the suitability of MIMO for resource-constrained control pipelines. 
\end{abstract}

\section{Introduction}\label{sec:intro}

Modern control and robotics systems routinely fuse heterogeneous sensing channels, including cameras, microphones, lidar/radar, and natural-language or operator signals. In these systems, multi-modal sensor learning (MSL) serves as an upstream estimation and representation module whose training dynamics strongly affect downstream control quality and robustness. Unlike uni-modal sensor pipelines, multi-modal sensor fusion learning can leverage complementary information across sensing channels, improving state inference and task performance under partial observability. This advantage has made multi-modal sensor fusion architectures central to embodied AI and autonomous decision-making stacks \citep{wang2022tag, shridhar2020alfworld, zhang2019neural, reed2022generalist}.

A common sensor-fusion approach uses separate encoders for each sensor modality to transform data into feature representations that are then ``fused'' before further processing. The fused representation is subsequently processed by a model head to produce the output. In the two-modality case, the multi-modal sensor learning problem can be formulated as:
\begin{equation}
    \min_{\vartheta_{mm}, \theta_{m_1}, \theta_{m_2}} f_{mm}(\vartheta_{mm}, \theta_{m_1}, \theta_{m_2}) \label{eq:multi-modal-intro}
\end{equation}
where $f_{mm}$ is the multi-modal loss, $\theta_{m_1}$ and $\theta_{m_2}$ are encoders for modalities $m_1$ and $m_2$, and $\vartheta_{mm}$ is the fusion head that processes the fused features. The goal is then to optimize $f_{mm}$ in \eqref{eq:multi-modal-intro} using standard optimization algorithms (e.g., SGD).
\begin{figure*}[t]
    \centering
\includegraphics[width=\linewidth]{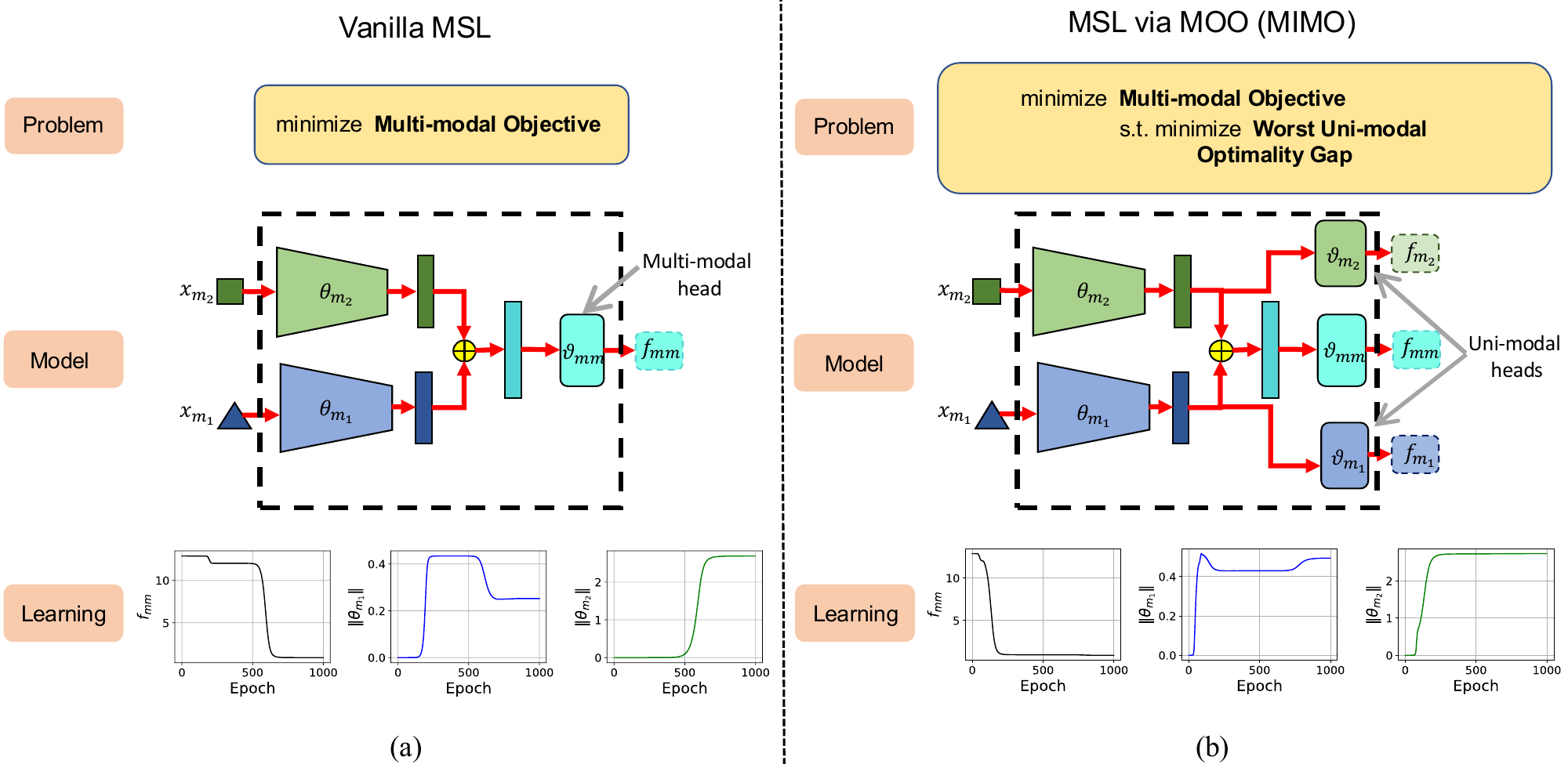}
    \caption{\textbf{Balancing multi-modal sensor learning via multi-objective optimization.} \textbf{(a)} Optimizing the standard fusion objective can lead to slower convergence because fast-to-learn modalities dominate the optimization. \textbf{(b)} We propose MIMO (MSL via MOO), which optimizes a modified objective that prevents single-modality dominance, yielding faster convergence.
    }
    \label{fig:overview}
\end{figure*}

\subsection{Imbalance issue in multi-modal sensor learning} 
While multi-modal sensor learning can outperform uni-modal sensor learning, recent studies show that standard training may still underperform the best single modality \citep{wang2020makes, huang2022modality}. For control-oriented systems, this issue is particularly critical: if one sensing stream dominates the updates, the learned estimator can become brittle to disturbances, occlusions, or sensor faults at deployment time. Prior works \citep{ma2022multimodal, hu2022shape, peng2022balanced} attribute this behavior to dominant modalities suppressing slower but informative modalities during optimization. Using a toy example, Figure \ref{fig:overview}(a) illustrates this imbalance: $\theta_{m_1}$ and $\theta_{m_2}$ evolve at different rates, producing slower and biased convergence of the fused model.

Recent studies have provided theoretical explanations for this inefficiency by analyzing the training process of late-fusion models. In \citep{allen2020towards, huang2022modality, han2022trusted}, the emergence of a dominant modality is explained through the concept of modality competition. In \citep{zhang2023theory}, the modality dominance is attributed to statistical characteristics of multi-modal data. Beyond these theoretical studies, empirical works have developed methods for balanced learning of different modalities, thereby improving performance \citep{peng2022balanced, fujimori2020modality, yao2022modality, peng2022balanced, li2023boosting, wei2024innocent}. However, a principled algorithm with low per-iteration computational complexity for balanced multi-modal sensor learning, together with convergence guarantees, is still missing.

\subsection{Our contributions}
We address modality imbalance through a control-motivated optimization reformulation. The main idea is to shape training dynamics so that slow-to-learn sensing channels are not suppressed by dominant channels. We augment the fused objective with modality-specific objectives and enforce a lexicographic priority on the worst-performing uni-modal sensor channel. Concretely, we reformulate the multi-modal sensor learning problem in \eqref{eq:multi-modal-intro} as a lexicographic MOO problem \citep{miettinen1999nonlinear}, given by
\begin{align}
    &\min_{\vartheta_{mm}, \theta_{m_1}, \theta_{m_2}} f_{mm}(\vartheta_{mm}, \theta_{m_1}, \theta_{m_2}) \label{eq:mml-via-moo-intro}\\
    &\text{s.t} ~~\theta_{m_1}, \theta_{m_2}, \vartheta_{m_1}, \vartheta_{m_2} \nonumber\\ &~~~~~\in  { \arg\!\!\!\!\!\min_{ \substack{\theta_{m_1}, \theta_{m_2}\\ \vartheta_{m_1}, \vartheta_{m_2}}}\max_{ k\in\{1,2\}} \left( f_{m_k}(\vartheta_{m_k}, \theta_{m_k}) - f^*_{m_k} \right)}.\nonumber
\end{align}
where $f_{m_k}$ for $k\in\{1, 2\}$ are uni-modal sensor objectives induced by separate uni-modal sensor heads $\vartheta_{m_1}$ and $\vartheta_{m_2}$ (see Figure \ref{fig:overview}(b)), and $f^*_{m_k} = \min_{\vartheta_{m_k}, \theta_{m_k}} f_{m_k}(\vartheta_{m_k}, \theta_{m_k})$ for all $k\in\{1, 2\}$. The formulation in \eqref{eq:mml-via-moo-intro} extends naturally to an arbitrary number of modalities in \eqref{eq:mml-via-moo}; here we focus on the two-modality case for clarity.
We then propose MIMO (multi-modal sensor learning via MOO), a lightweight gradient-based algorithm for the resulting formulation. As shown in Figure \ref{fig:overview}(b), our method alleviates modality imbalance and converges quickly to the global optimum. Specifically, the weights of $\theta_{m_1}$ and $\theta_{m_2}$ are learned at similar rates (Figure \ref{fig:overview}(b), $\Vert \theta_{m_1}\Vert$ and $\Vert \theta_{m_2}\Vert$), leading to faster convergence.

Our main contributions are threefold. First, we introduce a lexicographic MOO reformulation of MSL that prioritizes worst-channel performance while preserving the original fused objective. Second, we develop MIMO, a simple first-order algorithm with low per-iteration overhead and no expensive inner-loop balancing subroutine. Third, we provide convergence guarantees and empirical evidence across standard benchmarks, showing improved balanced performance with up to $20\times$ lower subroutine time versus strong baselines.

\section{Preliminaries}\label{sec:background}

In this section, we introduce basic concepts in multi-modal sensor learning \citep{wang2020makes} and MOO \citep{miettinen1999nonlinear} used to build a formulation for balanced multi-modal sensor learning.

\subsection{Multi-modal sensor learning (MSL)}\label{sec:mml}

Consider the classification problem on a multi-modal sensor dataset $\mathcal{D}_{mm} := \{ x^{(m_1)}_i, x^{(m_2)}_i, \dots, x^{(m_K)}_i, y_i\}_{i=1}^N$, which consists of $N$ input samples $x^{(m_k)}_i$ from $K$ modalities $m_k$, where $k\!\in\![K] := \{1, 2, \dots, K\}$, together with the corresponding labels $y_i$. 
Unlike the uni-modal sensor case, we need a ``fusion'' strategy to combine inputs from different modalities before producing the output used in the loss function. 
Depending on whether fusion happens during feature extraction, after feature extraction, or in a hybrid manner, fusion strategies can be classified as {\em early, late, or hybrid fusion}, respectively \citep{li2023boosting}. In early fusion, data from different modalities are processed jointly starting from the raw inputs to obtain multi-modal features. In late fusion, separate encoders are used to extract uni-modal sensor features, which are then fused at the final stage of the model. Any combination of early and late fusion is referred to as hybrid fusion.
We consider the late-fusion setting, where different modalities are fused after extracting features from each modality (see Figure \ref{fig:overview}(a)). Common fusion operations include summing the extracted features (summation) or concatenating them (concatenation) \citep{peng2022balanced}.
Let $\theta_{m_k}$ denote the parameters of the encoder that extracts features from inputs $x^{(m_k)}_i$ for all $k\in[K]$, and $\vartheta_{mm}$ denote the parameters of the fusion head that maps the fused features to the target output. Then, the problem of finding the optimal multi-modal sensor fusion model can be formulated as
\begin{equation}
    \min_{\boldsymbol{\Theta_{mm}}\in \mathbb{R}^{d_{mm}}} f_{mm}(\vartheta_{mm}, \theta_{m_1}, \theta_{m_2}, \dots, \theta_{m_K}), \label{eq:multi-modal}
\end{equation}
where $\boldsymbol{\Theta_{mm}}:= [\vartheta_{mm}; \theta_{m_1}; \theta_{m_2}; \dots ; \theta_{m_K}]$, and $f_{mm}:\mathbb{R}^{d_{mm}}\mapsto\mathbb{R}$ is the multi-modal sensor fusion objective defined over all input modalities in $\mathcal{D}_{mm}$. Solving \eqref{eq:multi-modal} yields the optimal fusion model $\boldsymbol{\Theta^*_{mm}}$.

\textbf{Uni-modal bias in sensor-fusion training.} As described in Section \ref{sec:intro}, imbalance in sensor-fusion training can lead to poor fused performance, even relative to uni-modal sensor learning. To illustrate this issue, we use an example from \citep{zhang2023theory} involving a multi-modal regression problem with a two-layer fully connected network, where one layer serves as the encoder and the other serves as the fusion head, together with concatenation fusion. A toy implementation of this example is shown in Figure \ref{fig:overview}, and the corresponding implementation details are given in Appendix \ref{app:toy-illustration}.

Concretely, consider a two-modality dataset $\mathcal{D}_{mm} := \{ x^{(m_1)}_i, x^{(m_2)}_i , y_i\}_{i=1}^N$ with $x^{(m_1)}_i \in \mathbb{R}^{d_1}$, $x^{(m_2)}_i \in \mathbb{R}^{d_2}$, and $y_i \in \mathbb{R}$ for all $i\in[N]$. Let the empirical input and input-output correlation matrices for modality $m_1$ be $C_{m_1}$ and $C_{ym_1}$ (and similarly for modality $m_2$). Let the cross-correlation matrices between $m_1$ and $m_2$ be $C_{m_1m_2}$ and $C_{m_2m_1}$. The exact definitions of these matrices are given in Appendix \ref{app:toy-illustration}. 

\textbf{Illustration with the two-layer model.} For the two-layer fully connected late-fusion multi-modal network with concatenation fusion, we choose the uni-modal sensor encoder parameters as $\theta_{m_1}\in \mathbb{R}^{d_1 \times d_h}$ and $\theta_{m_2} \in \mathbb{R}^{d_2 \times d_h}$, where $d_h$ is the dimensionality of the encoder layer for each modality. Note that the fusion head $\vartheta_{mm}$ can be partitioned into modality-specific components $\vartheta_{mm, m_1}$ and $\vartheta_{mm, m_2}$ as $\vartheta_{mm} = [ \vartheta_{mm, m_1}; \vartheta_{mm, m_2}]$, with $\vartheta_{mm, m_1}, \vartheta_{mm, m_2} \in \mathbb{R}^{d_h}$. Hence, we denote $\boldsymbol{\Theta_{mm}} := [ \vartheta_{mm, m_1} ;  \vartheta_{mm, m_2};\theta_{m_1};  \theta_{m_2}]$. All model parameters are initialized close to zero.

For a given data index $i$, the output of the multi-modal sensor fusion model can then be given as
\begin{equation}
    \hat{y}_i = \vartheta_{mm, m_1}^\top\theta_{m_1}^\top x^{(m_1)}_i + \vartheta_{mm, m_2}^\top\theta_{m_2}^\top x^{(m_2)}_i.
\end{equation}
The decoupled nature of the output with respect to the modalities is a consequence of the late-fusion architecture with concatenation fusion. The corresponding regression loss is given by $f_{mm}(\boldsymbol{\Theta_{mm}}) = \frac{1}{N}\sum_{i=1}^{N}(y_i - \hat{y}_i)^2$. We can then derive the gradient of $f_{mm}$ as follows (see Appendix \ref{app:toy-illustration} for details):
{\small
\begin{equation}
    \nabla_{\boldsymbol{\Theta_{mm}}}f_{mm}(\boldsymbol{\Theta_{mm}}) = [ \theta_{m_1}^\top\Psi_1;  \theta_{m_2}^\top\Psi_2;  \Psi_1\vartheta_{mm, m_1}^\top ;  \Psi_2\vartheta_{mm, m_2}^\top],
\end{equation}
}
where
\begin{align}
    \Psi_1 &= C_{ym_1} - \vartheta_{mm, m_1}\theta_{m_1}C_{m_1} - \vartheta_{mm, m_2}\theta_{m_2}C_{m_2m_1},\nonumber\\
    \Psi_2 &= C_{ym_2} - \vartheta_{mm, m_1}\theta_{m_1}C_{m_1m_2} - \vartheta_{mm, m_2}\theta_{m_2}C_{m_2}.\nonumber
\end{align}
To ensure $\nabla_{\boldsymbol{\Theta_{mm}}}f_{mm}(\boldsymbol{\Theta_{mm}})=0$ (i.e., stationarity), it suffices to achieve some combination of $\Psi_1=0$, $\Psi_2=0$, $\theta_{m_k}=0$, and $\vartheta_{m_k}=0$. However, for the model to have ``learned'' modality $m_k$, the weights corresponding to that modality should in general be nonzero, i.e., $\theta_{m_k}\neq0$ and $\vartheta_{m_k}\neq0$. Thus, to reach a stationary point while learning both modalities, we ideally want model parameters that satisfy $\Psi_k=0$, $\theta_{m_k}\neq0$, and $\vartheta_{m_k}\neq0$ for all $k \in\{1, 2\}$.

\textbf{Superficial modality preference.} Since the model weights are initialized near zero, the model first visits a uni-modal sensor stationary point where $\Psi_k=0$, $\vartheta_{mm, m_{3-k}}=0$, and $\boldsymbol{\Theta}_{m_{3-k}}=0$ for some $k\in\{1, 2\}$, before eventually reaching the ideal stationary point that achieves $\Psi_1=\Psi_2=0$, $\vartheta_{mm, m_{3-k}}\neq0$, and $\boldsymbol{\Theta}_{m_{3-k}}\neq0$ for all $k \in\{1, 2\}$. Which stationary point is visited first depends on the dataset statistics \citep{zhang2023theory} (see Appendix \ref{app:toy-illustration} for more details). Furthermore, it can be shown that the uni-modal sensor stationary point visited first is decoupled from which modality contributes more to minimizing $f_{mm}$; this phenomenon is known as ``superficial modality preference'' \citep{zhang2023theory}. For example, in the toy MSL task shown in Figure \ref{fig:overview}, the drop in the objective value $f_{mm}$ is smaller when modality $m_1$ is learned (when the norm of the encoder weights $\Vert \theta_{m_1} \Vert$ attains a nonzero value) than when modality $m_2$ is learned, even though modality $m_1$ is learned first. Thus, imbalanced modality learning can lead to models that overfit a fast-learning modality and are therefore suboptimal for the overall multi-modal sensor fusion objective.

\subsection{Lexicographic MOO}

In this section, we introduce the key MOO formulation that we use to alleviate the imbalance issue in MSL discussed in Section \ref{sec:mml}. Specifically, we introduce lexicographic MOO, which is useful when one can assign an \emph{a priori} order to the learning objectives based on preference. In other words, optimizing lower-priority objectives is constrained by the optimality of higher-priority objectives. More concretely, consider a set of objectives $f_m:\mathbb{R}^d\mapsto\mathbb{R}$ for $m\in[M]$. Then, lexicographic MOO problem can be formulated as \citep{miettinen1999nonlinear}
\begin{equation}
    \text{lex}\min_{\boldsymbol{\Theta} \in \mathbb{R}^d} F^{\rm Lex}(\boldsymbol{\Theta}) := f_1(\boldsymbol{\Theta}), f_2(\boldsymbol{\Theta}), \dots, f_M(\boldsymbol{\Theta}) \label{eq:lex-prob-form}
\end{equation}
where the index of the objectives specifies the order in which the optimality of each objective should be achieved. For the bi-objective case, \eqref{eq:lex-prob-form} can be written as the constrained optimization problem
\begin{equation}\label{eq:bi-lex-prob-form}
    \min_{\boldsymbol{\Theta} \in \mathbb{R}^d}  ~~f_2(\boldsymbol{\Theta}) ~~~ \text{s.t.} ~~ \boldsymbol{\Theta} \in \arg\min_{\boldsymbol{\Theta} \in \mathbb{R}^d} f_1(\boldsymbol{\Theta}).  
\end{equation}
Note that this method allows us to incorporate prior knowledge about the problem into the optimization process. It can be shown that the solution of \eqref{eq:lex-prob-form} is Pareto optimal \citep{miettinen1999nonlinear}.

\section{Balanced MSL via MOO}\label{sec:balanced-mml-via-moo}

In this section, we first present our proposed reformulation of the multi-modal sensor learning problem to address the imbalance issue discussed in the previous section. We then detail the corresponding algorithm and provide a convergence analysis for the proposed method.

\subsection{Problem formulation}

In this section, we modify the original multi-modal sensor learning problem \eqref{eq:multi-modal} to encourage balanced learning across modalities. Intuitively, the slow-to-learn modality will have the largest optimality gap. To alleviate this imbalance, we modify the MSL formulation to encourage learning of the modality with the worst optimality gap. Specifically, we seek to optimize the multi-modal sensor fusion objective subject to the optimality of the worst-performing (in terms of optimality gap) uni-modal sensor objective, namely,
\begin{align}
    &\min_{\boldsymbol{\Theta_{mm}}\in \mathbb{R}^{d_{mm}}} f_{mm}(\boldsymbol{\Theta_{mm}} ) \label{eq:mml-via-moo} \\&~~\text{s.t} ~~\{\boldsymbol{\Theta_{m_k}}\}_{k=1}^K \in \arg\!\!\min_{\{\boldsymbol{\Theta_{m_k}'}\}_{k=1}^K} \max_{k\in[K]} \left( f_{m_k}(\boldsymbol{\Theta_{m_k}'}) - f^*_{m_k} \right), \nonumber
\end{align}
where $\boldsymbol{\Theta_{m_k}} = [\vartheta_{m_k}; \theta_{m_k}]$ with $\vartheta_{m_k}$ denoting the uni-modal sensor head dedicated to modality $m_k$, and $f^*_{m_k} = \min_{\boldsymbol{\Theta_{m_k}}} f_{m_k}(\boldsymbol{\Theta_{m_k}})$ for all $k\in[K]$. Intuitively, formulation \eqref{eq:mml-via-moo} learns the multi-modal sensor head by optimizing the fusion objective $f_{mm}$ over the backbone model $\theta_{m_k}$, which optimizes the worst-performing modality $\max_{{\scriptscriptstyle k\in [K]}} \left( f_{m_k}(\boldsymbol{\Theta_{m_k}}) - f^*_{m_k} \right)$.

\begin{Remark}
Note that \eqref{eq:mml-via-moo} follows the lexicographic MOO structure in \eqref{eq:bi-lex-prob-form}, where $f_2(\boldsymbol{\hat{\Theta}_{mm}})\!\! = \!\!f_{mm}(\boldsymbol{\Theta_{mm}})$ and $f_1(\boldsymbol{\hat{\Theta}_{mm}})\!\! =\!\! \max_{k\in[K]} \left( f_{m_k}(\boldsymbol{\Theta_{m_k}}) - f^*_{m_k} \right)$ with $\boldsymbol{\hat{\Theta}_{mm}}:= [\vartheta_{mm}; \vartheta_{m_1}; \dots ; \vartheta_{m_K}; \theta_{m_1}; \dots ; \theta_{m_K}]$. However, unlike in \eqref{eq:bi-lex-prob-form}, in \eqref{eq:mml-via-moo} only part of $\boldsymbol{\hat{\Theta}_{mm}}$ (the set of uni-modal sensor encoders) is shared between the two objectives, allowing for independent optimization of the non-shared part of $\boldsymbol{\hat{\Theta}_{mm}}$ (multi- and uni-modal sensor heads).
\end{Remark}
The only parameters shared between the uni-modal sensor and multi-modal sensor fusion objectives are the uni-modal sensor encoders $\theta_{m_k}$ for all $k\in[K]$. 
Note that the lexicographic optimization of the worst-performing uni-modal sensor objective and the multi-modal sensor fusion objective with respect to the shared parameters $\theta_{m_k}$ for $k\in[K]$ can be viewed as a simple bilevel optimization problem. 
Recently, \citep{shen2023penalty} introduced a reformulation of bilevel optimization as a single-level problem by penalizing lower-level constraints in the upper-level objective.
Leveraging this view, we can obtain the single level MIMO objective by rewriting \eqref{eq:mml-via-moo} as
\begin{align}\label{eq:mimo} 
 \min_{\boldsymbol{\hat{\Theta}_{mm}}\in \mathbb{R}^{\hat{d}_{mm}}} \hat{f}_{mm}(\boldsymbol{\hat{\Theta}_{mm}}) 
    &:= f_{mm}(\boldsymbol{\Theta_{mm}})\\ &~~~+\lambda \max_{{\scriptscriptstyle k\in[K]}} \left( f_{m_k}(\boldsymbol{\Theta_{m_k}}) - f^*_{m_k} \right), \nonumber
\end{align}

where $\boldsymbol{\hat{\Theta}_{mm}} = [\vartheta_{mm}; \{\vartheta_{m_1}\}_{k=1}^K; \{\theta_{m_k}\}_{k=1}^K]$, and $\lambda>0$ is a problem dependent penalty parameter.  

\textbf{How does MIMO alleviate modality imbalance?} To understand how this new formulation alleviates modality imbalance, let us revisit the two-layer fully connected late-fusion multi-modal network with concatenation fusion, for the two modality case ($K=2$), now augmented with additional uni-modal sensor heads (linear layers) to obtain uni-modal sensor objectives $f_{m_1}$ and $f_{m_2}$ (see Figure \ref{fig:overview}). Let the weight vectors corresponding to the uni-modal sensor heads be $\vartheta_{m_1}\in \mathbb{R}^{d_h}$ and $\vartheta_{m_2}\in \mathbb{R}^{d_h}$. Then, the new parameter vector can be written as $\boldsymbol{\hat{\Theta}_{mm}} := [\vartheta_{mm, m_1} ;  \vartheta_{mm, m_2}; \vartheta_{m_1}; \vartheta_{m_2};  \theta_{m_1};  \theta_{m_2}]$. The corresponding uni-modal sensor objectives are given by $f_{m_k}(\boldsymbol{\Theta_{m_k}}) = \frac{1}{N}\sum_{i=1}^{N}(y_i - \hat{y}^{(m_k)}_i)^2$ for $k\in[2]$, where $\hat{y}^{(m_k)}_i = \vartheta_{m_k}^\top\theta_{m_k}^\top x^{(m_k)}_i$ for each data index $i\in[N]$. With these definitions, we can rewrite the gradients of each layer of the network with additional uni-modal sensor heads as
{\small
\begin{align}\label{eq:fhat-grad}
    \!\!\!\!\!\nabla_{\boldsymbol{\hat{\Theta}_{mm}}}\hat{f}_{mm}(\boldsymbol{\hat{\Theta}_{mm}}) &= [ \theta_{m_1}^\top(\Psi_1 + \lambda_1\check{\Psi}_1 );  \vartheta_{mm, m_2}^\top(\Psi_2+ \lambda_2\check{\Psi}_2 ); \nonumber\\
    &\Psi_1\vartheta_{mm, m_1}^\top ;  \Psi_2\vartheta_{mm, m_2}^\top;\lambda_1\check{\Psi}_1\theta_{m_1}^\top; \lambda_2\check{\Psi}_2\theta_{m_2}^\top],
\end{align}}
where $\check{\Psi}_1 \!\!=\!\! C_{ym_1}\!\! - \vartheta_{m_1}\theta_{m_1}C_{m_1}$, $\check{\Psi}_2 \!\!=\!\! C_{ym_2} \!- \vartheta_{m_2}\theta_{m_2}C_{m_2}$, and $\lambda_i\!=\!\lambda$ if $i\!=\!\arg\!\!\!\max\limits_{{\scriptscriptstyle k\in\{1,2\}}} \left( f_{m_k}(\boldsymbol{\Theta_{m_k}}) - f^*_{m_k} \right)$, and $0$ otherwise.
This modification, introduced by incorporating modality-specific objectives, helps balance multi-modal sensor learning in an intuitive way. Suppose modality $m_1$ is quick to learn. Then, initially, the weights in the $m_2$ component of the model remain close to zero, so  optimality gap of $m_2$ is larger than that of $m_1$. As a result, $\lambda_1=0$ and $\lambda_2=\lambda$. Due to the amplified gradient induced by the $\lambda$ weighting, the weights in the $m_2$ component are updated more rapidly through the gradient contribution of the $f_{m_2}$ objective, while modality $m_1$ is learned through the gradient contribution of $f_{mm}$. At stationarity, since all gradient components in \eqref{eq:fhat-grad} must be zero, the gradient components contributed by $f_{mm}$ and $f_{m_2}$ must each vanish. Hence, the model achieves optimality for the original multi-modal sensor fusion objective $f_{mm}$ while simultaneously learning modality $m_2$.

\textbf{Formal justification of penalty reformulation. } 
To provide more insights on why optimizing MIMO in \eqref{eq:mimo} alleviates the imbalanced modality issue, we formally prove the approximate equivalence between the penalty reformulation \eqref{eq:mimo} with the worst-case MOO problem \eqref{eq:mml-via-moo}. 

\begin{Assumption}[PL condition of objectives]\label{ass:PL}
We assume that each modality objective $f_{m_k}(\boldsymbol{\Theta_{m_k}})$ satisfies the Polyak-Łojasiewicz (PL) condition with $\mu$, for all $k\in [K]$, i.e. 
\begin{align*}
\|\nabla f_{m_k}(\boldsymbol{\Theta_{m_k}})\|^2\geq\mu (f_{m_k}(\boldsymbol{\Theta_{m_k}})-f_{m_k}^*). 
\end{align*}
\end{Assumption} 
PL condition is commonly used in bilevel optimization as a generalization of strongly convexity \citep{shen2023penalty,xiao2023generalized,kwon2023penalty,chen2024finding}, and holds for various non-convex applications such as linear quadratic regulator models \citep{fazel2018global}, phase retrieval \citep{sun2018geometric}, over-parameterized neural network \citep{liu2022loss}, including convolutional neural network \citep{chenover}, residual network \citep{marion2023implicit}, etc. Next, we introduce the notion of smooth functions.

\begin{Definition}\label{def:smooth}
    A differentiable function $g:\R^d\rightarrow \R$ is $L$-smooth if for all $\boldsymbol{\Theta}_1, \boldsymbol{\Theta}_2\in\R^d$, the gradient of $g$ satisfies the condition
    \begin{equation}
        \|\nabla g(\boldsymbol{\Theta}_1) - \nabla g(\boldsymbol{\Theta}_2) \| \leq L \| \boldsymbol{\Theta}_1 - \boldsymbol{\Theta}_2\|.
    \end{equation}
\end{Definition}


\begin{Assumption}[Smoothness of objectives]\label{ass:smooth}
    Functions $f_{mm}(\boldsymbol{\Theta_{mm}})$ and $f_{m_k}(\boldsymbol{\Theta_{m_k}})$ are $L_{mm}$- and $L_{m_k}$-smooth (Definition \ref{def:smooth}), respectively, where $k\in[K]$.
\end{Assumption}
Assumption \ref{ass:smooth} is standard in the optimization \citep{nesterov2018lectures}.  Under above assumptions, we can prove that the solutions of penalty problem \eqref{eq:mimo} with certain penalty constant $\lambda$ are also the solutions to the $\epsilon$-relaxed worst-case problem \eqref{eq:mml-via-moo} as follows: 
\begin{align}
    &\min_{\boldsymbol{\Theta_{mm}}\in \mathbb{R}^{d_{mm}}} f_{mm}(\boldsymbol{\Theta_{mm}} ) \label{eq:mml-via-moo-relax} \\&~~\text{s.t} ~~\max_{k\in[K]} \left( f_{m_k}(\boldsymbol{\Theta_{m_k}}) - f^*_{m_k} \right)\leq \epsilon, \nonumber 
\end{align} 
Here, the constraint in \eqref{eq:mml-via-moo} enforces that all modalities be learned equally well and requires the worst-case modality $f_{m_k}$ to stay close to its best achievable performance $f^*_{m_k}$, thereby preventing it from being overlooked too much. We will then show that solving the MIMO problem~\eqref{eq:mimo} inherently addresses the relaxed worst-case problem \eqref{eq:mml-via-moo}. 

\begin{Theorem}[Relations of the penalty reformulation and the relaxed problem]\label{thm:penalty_equ}
Under Assumption \ref{ass:PL}--\ref{ass:smooth}, let $(\boldsymbol{\Theta_{mm}^*}, \boldsymbol{\Theta_{m_k}^*})$ be an $\epsilon_\lambda$-global (or local) solution of the penalty problem ~\eqref{eq:mimo} with $\lambda$. Then if $\lambda = \mathcal{O}(\epsilon_\lambda^{-1})$, then $(\boldsymbol{\Theta_{mm}^*}, \boldsymbol{\Theta_{m_k}^*})$ is an $\epsilon_\lambda$-global (or local) solution to the relaxed problem in~\eqref{eq:mml-via-moo-relax}, with some $\epsilon = \mathcal{O}(\epsilon_\lambda^{2})$. 
\end{Theorem} 

The proof of Theorem \ref{thm:penalty_equ} is deferred to Appendix \ref{app:penalty-equ-proof}, which relies on the semi-smoothness and the nonsmooth PL property of $\max_{k\in[K]} \left( f_{m_k}(\boldsymbol{\Theta_{m_k}}) - f^*_{m_k} \right)$, and generalizes the results in \citep{shen2023penalty} to the nonsmooth setting. Theorem \ref{thm:penalty_equ} shows that solving the MIMO problem \eqref{eq:mimo} is equivalent to solving the $\epsilon$-relaxed worst-case problem \eqref{eq:mml-via-moo-relax}, which motivates us to develop algorithm to solve MIMO problem \eqref{eq:mimo} to alleviate the learning unbalanced issue for the worst-case modality. 

\subsection{Algorithm development}

In this section, we present the algorithm used to solve the problem in \eqref{eq:mimo}. First, note that the penalty term in \eqref{eq:mimo} is nonsmooth because of the $\max$ operator. In general, this kind of nonsmoothness leads to the slower convergence rate $\mathcal{O}(1/\epsilon^2)$ for subgradient-based optimization, compared to $\mathcal{O}(1/\epsilon)$ for smooth gradient-based optimization \citep{nesterov2005smooth}. To alleviate this issue, prior work has proposed using smoothing functions \citep{lin2024smooth} that encode structure in the nonsmooth term, yielding convergence to an $\mathcal{O}(\epsilon)$-suboptimal point of the original nonsmooth problem within $\mathcal{O}(1/\epsilon)$ iterations. We next introduce the notion of smoothing functions.

\begin{algorithm}[t]
\caption{MIMO (multi-modal sensor learning via MOO) 
}\label{alg:mml-via-moo} 
\begin{algorithmic}
\STATE\textbf{input} $\hat{\boldsymbol{\Theta}}_{mm,1}:=[\vartheta_{mm,1}; \vartheta_{m_1,1}; \vartheta_{m_2,1}; \theta_{m_1,1}; \theta_{m_2,1}]$, learning rates $\{\eta_t\}_{t=1}^T$, penalty parameter $\lambda$, smoothing parameter $\mu$
\FOR {$t=1, \dots, T$}
    \STATE Compute gradient of $\hat{f}_{mm}$ given in \eqref{eq:penalty-smooth-cheb-two}
    \STATE Update $\hat{\boldsymbol{\Theta}}_{mm,t+1}$ following \eqref{eq:mimo-update}
\ENDFOR
\STATE \textbf{output} $\hat{\boldsymbol{\Theta}}_{mm,T+1}$
\end{algorithmic} 
\end{algorithm}

\begin{figure*}[t]
    \includegraphics[align=c,width=2.5in]{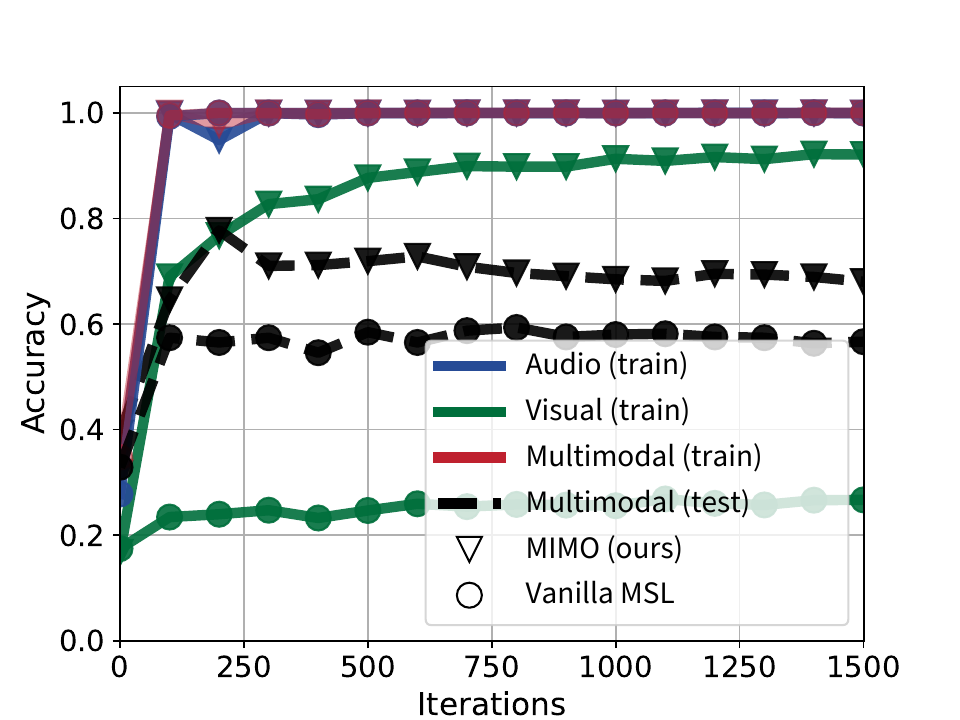}
    \includegraphics[align=c,width=1.74in]{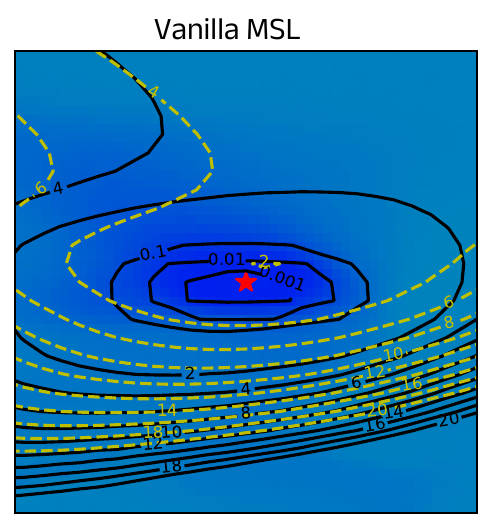}
    \includegraphics[align=c,width=2.35in]{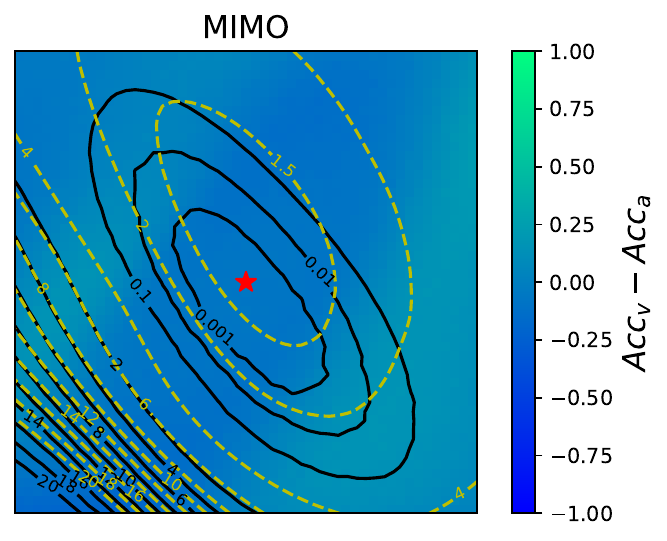}
    \caption{ 
 \textbf{Left:} Comparison of the training and testing performance of MIMO algorithm with vanilla MSL (joint training with sum fusion) on CREMA-D dataset. 
    \textbf{Middle and Right:} Comparison of the loss landscape of vanilla MSL and MIMO after 1500 iterations on CREMA-D dataset. The black contours (\textbf{---})
    denote the multi-modal training loss, and the yellow dashed contours 
    ({\color{yellow} \textbf{- - -}})
    denote the multi-modal testing loss. The red star
    ({\color{red} \ding{72}}) denotes the convergent point of each method.
    The color of the heatmap
    denotes the difference between uni-modal sensor training accuracies at the given point of the loss landscape, where blue (\raisebox{0.5ex}{\colorbox{blue}{\makebox[0.1cm][c]{}}}) denotes audio modality is dominating, green (\raisebox{0.5ex}{\colorbox{green}{\makebox[0.1cm][c]{}}}) denotes visual modality is dominating, and higher color intensity denotes larger differences in accuracy. 
    As illustrated by the training curves 
    and loss landscapes, 
    MIMO achieves lower multi-modal test loss (i.e. better generalization) by balancing the learning of each modality.
    }\label{fig:cremad-illustration}
\end{figure*}

\begin{Definition}\label{def:smoothing-func}
    For a continuous function $g:\R^d\rightarrow \R$, we call $g_\mu:\R^d\rightarrow \R$ a smoothing function of $g$ if for any $\mu>0$, $g_\mu$ is continuously differentiable in $\R^d$ and satisfies the conditions (1) $\lim\limits_{\boldsymbol{\Phi}\rightarrow\boldsymbol{\Theta}, \mu \downarrow 0}g_\mu(\boldsymbol{\Phi}) = g(\boldsymbol{\Theta})$; and (2) there exists constants $L$ and $\alpha>0$ independent of $\mu$, such that $g_\mu$ is $\left(L+\alpha \mu^{-1}\right)$-smooth (Definition \ref{def:smooth}).
\end{Definition}
Note that $g_\mu(\boldsymbol{\Theta}) := \mu\log\left(\sum_{m=1}^M\exp\left(\mu^{-1}g_{m}(\boldsymbol{\Theta}) \right)\right)$ is a smoothing function for $g(\boldsymbol{\Theta}) = \max_{m\in[M]} g_{m}(\boldsymbol{\Theta})$ \citep{lin2024smooth}, where $\mu>0$ is the smoothing parameter that controls the smoothness of $g_\mu$. Then, the smoothed version of \eqref{eq:mimo} is
\begin{align}\label{eq:penalty-smooth-cheb-two}
 &\min_{\boldsymbol{\hat{\Theta}_{mm}}\in \mathbb{R}^{\hat{d}_{mm}}} \hat{f}_{mm, \mu}(\boldsymbol{\hat{\Theta}_{mm}})
    := f_{mm}(\vartheta_{mm}, \theta_{m_1}, \theta_{m_2})\\
    &\quad\quad+\lambda \mu\log\left(\sum_{k=1}^2\exp\left(\mu^{-1}(f_{m_k}(\vartheta_{m_k}, \theta_{m_k}) - f^*_{m_k}) \right)\right).\nonumber
\end{align}    
With this formulation, we can apply gradient descent to the smoothed MIMO objective $\hat{f}_{mm, \mu}$ in \eqref{eq:penalty-smooth-cheb-two} as
\begin{align}\label{eq:mimo-update}
    \hat{\boldsymbol{\Theta}}_{mm,t+1} \!=\! \hat{\boldsymbol{\Theta}}_{mm,t} - \eta_t \nabla \hat{f}_{mm, \mu}(\hat{\boldsymbol{\Theta}}_{mm, t}),
\end{align}
where $t$ is the iteration index and $\eta_t$ is the learning rate. Specifically, for $K=2$, the MIMO update in \eqref{eq:mimo-update} can be decomposed into updates for each component of $\hat{\boldsymbol{\Theta}}_{mm}$ as
\begin{align*}
    \vartheta_{mm, t+1} &= \vartheta_{mm, t} - \eta_t \nabla_{\vartheta_{mm}}f_{mm}(\vartheta_{mm, t}, \theta_{m_1, t}, \theta_{m_2, t})\\
    \vartheta_{m_k, t+1} &= \vartheta_{m_k, t} - \eta_t \lambda \sigma_{m_k}\nabla_{\vartheta_{m_k}}f_{m_k}(\vartheta_{m_k, t}, \theta_{m_k, t})\\
    \theta_{m_k, t+1} &= \theta_{m_k, t} - \eta_t \nabla_{\theta_{m_k}}f_{mm}(\vartheta_{mm, t}, \theta_{m_1, t}, \theta_{m_2, t}) \\&-\eta_t \lambda \sigma_{m_k} \nabla_{\theta_{m_k}}f_{m_k}(\vartheta_{m_k, t}, \theta_{m_k, t})
\end{align*}
where {\small $\sigma_{m_k}\!\!:=\!\!\exp(h_{\mu, k})/{\sum_{k'=1}^2\exp (h_{\mu, k'})}$}, $h_{\mu, k}=\mu^{-1}(f_{m_k}(\vartheta_{m_k}, \theta_{m_k}) - f^*_{m_k})$, and $k\in\{1, 2\}$. The MIMO algorithm is summarized in Algorithm \ref{alg:mml-via-moo}. Importantly, the careful construction of a single objective for balanced multi-modal sensor learning allows us to use gradient descent directly on $\hat{f}_{mm, \mu}$, without computationally intensive heuristic subroutines that lack provable convergence guarantees. Next, we make the following assumptions to establish convergence of Algorithm \ref{alg:mml-via-moo}.

\begin{table*}[t]
{\small
\begin{center}
    \setlength{\tabcolsep}{0.15em} 
\caption{Multi-modal and uni-modal sensor test accuracy performance (Acc, Acc$_a$, Acc$_v$, Acc$_t$ ) of different MSL and MOO methods on the CREMA-D and UR-Funny datasets. $t(s)$ denotes the average subroutine time for each method. The best (highest) accuracy results are shown in \textbf{bold}. The best (lowest) subroutine time among the first three best-performing methods (in Acc) is \underline{underlined}. All error values denote one standard deviation.}\label{tab:cremad-urfunny}
\begin{tabular}{lccccccccc}
\hline
\multirow{2}{*}{Method} & \multicolumn{4}{c}{CREMA-D} & \multicolumn{5}{c}{UR-Funny} \\
\cline{2-10}
& Acc {\tiny(\%)} & Acc$_a$ {\tiny(\%)}& Acc$_v$ {\tiny(\%)} & $t(s)$ & Acc {\tiny(\%)} & Acc$_a$ {\tiny(\%)} & Acc$_v$ {\tiny(\%)} & Acc$_t$ & $t(s)$ \\
\hline
\textbf{Audio}    &  - & 59.31{\tiny $\pm$ 0.76} & - & 0.028{\tiny $\pm$ 0.005}
&  - & \textbf{58.00}{\tiny $\pm$ 0.74} & - & - & 0.037{\tiny $\pm$ 0.002}\\
\textbf{Visual}    &  - & - & \textbf{61.04}{\tiny $\pm$ 0.87} & 0.029{\tiny $\pm$ 0.005}
&- & - & \textbf{53.12}{\tiny $\pm$ 0.52} & - &  0.038{\tiny $\pm$ 0.002}  \\
\textbf{Text}  &  - & - & - & - 
&  - & - & - & \textbf{63.74}{\tiny $\pm$ 1.59} & 0.038{\tiny $\pm$ 0.002}\\
\textbf{ MSL}  &  60.26{\tiny $\pm$ 0.84} & 58.02{\tiny $\pm$ 0.58} & 22.69{\tiny $\pm$ 1.72}  & 0.038{\tiny $\pm$ 0.009}
&    63.10{\tiny $\pm$ 0.59} & 53.02{\tiny $\pm$1.16} & 50.18{\tiny $\pm$ 0.78} & 62.49{\tiny $\pm$ 0.90} & 0.038{\tiny $\pm$ 0.002}\\
\hline
\textbf{MSES}   &  57.96{\tiny $\pm$ 0.42} & 55.84{\tiny $\pm$ 0.91} & 27.37{\tiny $\pm$ 1.12}  & 0.042{\tiny $\pm$ 0.007}
&  62.90{\tiny $\pm$ 0.68} & 53.33{\tiny $\pm$1.17} & 49.91{\tiny $\pm$ 1.46} & 62.71{\tiny $\pm$ 0.60} & 0.071{\tiny $\pm$ 0.006}  \\
\textbf{MSLR}   &  62.09{\tiny $\pm$ 0.15} & 58.35{\tiny $\pm$ 0.62} & 25.62{\tiny $\pm$ 1.43}  & 0.0412{\tiny $\pm$ 0.006}
&  63.16{\tiny $\pm$ 0.45} & 54.77{\tiny $\pm$1.51} & 50.69{\tiny $\pm$ 0.18} & 61.89{\tiny $\pm$ 0.99} & 0.074{\tiny $\pm$ 0.012} \\
\textbf{OGM-GE}  &   74.49{\tiny $\pm$ 0.65} & 53.78{\tiny $\pm$ 1.21} & 47.82{\tiny $\pm$ 1.51}  & 0.112{\tiny $\pm$ 0.009}
&    - & - & - & - & -  \\
\textbf{AGM}   &  46.63{\tiny $\pm$ 0.93} & 43.05{\tiny $\pm$ 0.99} & 18.32{\tiny $\pm$ 1.11}  & 0.205{\tiny $\pm$ 0.005}  
&  64.18{\tiny $\pm$ 0.77} & 54.76{\tiny $\pm$0.65} & 49.45{\tiny $\pm$ 0.90} & 62.74{\tiny $\pm$ 0.76} & 0.384{\tiny $\pm$ 0.001}  \\
\textbf{EW}     &  65.50{\tiny $\pm$ 0.50} & 58.80{\tiny $\pm$ 0.77} & 59.66{\tiny $\pm$ 1.56}  & 0.036{\tiny $\pm$ 0.006}
&  63.63{\tiny $\pm$ 0.42} & 54.00{\tiny $\pm$ 0.94} & 49.69{\tiny $\pm$ 0.95}  & 62.41{\tiny $\pm$ 0.55}  & 0.090{\tiny $\pm$ 0.002} \\
\textbf{MGDA}     & 63.47{\tiny $\pm$ 0.79} & \textbf{60.80}{\tiny $\pm$ 0.68} & 26.25{\tiny $\pm$ 1.19}  & 0.310{\tiny $\pm$ 0.053}
&  63.81{\tiny $\pm$ 0.53} & 53.85{\tiny $\pm$ 1.37} & 49.81{\tiny $\pm$ 0.80}  & 62.96{\tiny $\pm$1.00}  & 0.441{\tiny $\pm$ 0.003} \\
\textbf{MMPareto}     &  68.67{\tiny $\pm$ 0.97} & 60.59{\tiny $\pm$ 0.57} & 58.82{\tiny $\pm$ 1.61}  & 0.309{\tiny $\pm$ 0.053}
&  63.94{\tiny $\pm$ 0.53} & 52.48{\tiny $\pm$ 1.57} & 50.31{\tiny $\pm$ 0.76}  & 63.35{\tiny $\pm$ 1.27}  & 0.436{\tiny $\pm$ 0.006} \\
\textbf{MIMO}     &  \textbf{75.96}{\tiny $\pm$ 0.83} & 55.60{\tiny $\pm$ 1.54} & 59.76{\tiny $\pm$ 1.40}  & \underline{0.037}{\tiny $\pm$ 0.012} 
&  \textbf{64.54}{\tiny $\pm$ 0.86} & 52.19{\tiny $\pm$ 1.08} & 50.38{\tiny $\pm$ 0.48}  & 62.10 {\tiny $\pm$ 1.31}  & \underline{0.101}{\tiny $\pm$ 0.013} \\
\bottomrule
\end{tabular}
\end{center}}
\vspace{-0.5cm}
\end{table*}

\begin{Assumption}[Lipschitz continuity of objectives]\label{ass:lip}
    For all $k\in[K]$, the objectives $f_{m_k}$ are $L_{m_k, 1}$-Lipschitz continuous, i.e., for any $\boldsymbol{\Theta_{m_k}}, \boldsymbol{\Theta_{m_k}'}$,
    \begin{equation}
        |f_{m_k}(\boldsymbol{\Theta_{m_k}}) - f_{m_k}(\boldsymbol{\Theta_{m_k}'})| \!\leq\! L_{m_k, 1} \Vert \boldsymbol{\Theta_{m_k}} - \boldsymbol{\Theta_{m_k}'} \Vert.
    \end{equation}
\end{Assumption}
Assumption \ref{ass:lip} is required in our analysis to establish the smoothness of the composite objective $\hat{f}_{mm, \mu}$. We can then state the following proposition. 

\begin{Proposition}[Smoothness of $\hat{f}_{mm, \mu}$]\label{prop:f-hat-smooth}
    Under Assumptions \ref{ass:smooth} and \ref{ass:lip}, there exist $\hat{L}_{mm}>0$ such that $\hat{f}_{mm, \mu}$ defined in \eqref{eq:penalty-smooth-cheb-two} is $\hat{L}_{mm}$-smooth (Definition \ref{def:smooth}), where $\hat{L}_{mm} := L_{mm} + \lambda \sum_{k=1}^2 \left(L_{m_k} + \mu^{-1}L^2_{m_k, 1} \right)$.
\end{Proposition}
A proof of Proposition \ref{prop:f-hat-smooth} is given in Appendix \ref{app:f-hat-smooth-proof}. Proposition \ref{prop:f-hat-smooth} establishes the smoothness of the objective used by Algorithm \ref{alg:mml-via-moo} for gradient descent. Standard gradient descent theory \citep{nesterov2018lectures} then yields the following result. 

\begin{Theorem}[Convergence]\label{thm:convergence}
Let Assumptions \ref{ass:smooth} and \ref{ass:lip} hold. For any $\lambda,\mu > 0$, and $0< \eta_t\leq 1/\hat{L}_{mm}$ for all $t \in [T]$, Algorithm \ref{alg:mml-via-moo} reaches an $\epsilon$-stationary point of $\hat{f}_{mm, \mu}$ with iteration complexity $\mathcal{O}(1/\epsilon)$. 
\end{Theorem}

\section{Experiments} \label{sec:experiments} 

In this section, we evaluate MIMO on several multi-modal benchmarks and compare it with popular multi-modal and MOO baselines. In addition to the performance comparison, we include experiments illustrating the generalization benefits of MIMO and an ablation study of its parameters. 

\textbf{Experiment settings.} We adopt the same experimental settings as \citep{li2023boosting} and \citep{peng2022balanced} across several popular multi-modal datasets. The \textbf{CREMA-D} dataset \citep{cao2014crema} is an audio-visual dataset for speech emotion recognition with six emotion labels. The \textbf{UR-Funny} dataset \citep{hasan2019ur} was created for humor detection and involves word (text), gesture (vision), and intonational cue (acoustic) modalities. \textbf{Kinetics-Sound} \citep{arandjelovic2017look} comprises 31 human action classes derived from the Kinetics dataset \citep{kay2017kinetics}, which contains 400 categories of YouTube videos with both audio and visual components. \textbf{VGGSound} \citep{chen2020vggsound} is a large-scale video dataset with 309 classes spanning a broad range of everyday audio events. Additional experimental details and results for \textbf{AV-MNIST}, \textbf{AVE}, and \textbf{CMU-MOSEI} are provided in Appendix \ref{app:experiments}. We compare MIMO with several popular multi-modal baselines, including Modality-Specific Early Stopping (\textbf{MSES}) \citep{fujimori2020modality}, Modality-Specific Learning Rate (\textbf{MSLR}) \citep{yao2022modality}, On-the-fly Gradient Modulation Generalization Enhancement (\textbf{OGM-GE}) \citep{peng2022balanced} (designed only for the two-modality case), and Adaptive Gradient Modulation (\textbf{AGM}) \citep{li2023boosting}. In addition to these baselines, we compare against MOO baselines such as equal weighting (\textbf{EW}), the Multiple Gradient Descent Algorithm (\textbf{MGDA}) \citep{Desideri2012mgda}, and \textbf{MMPareto} \citep{wei2024innocent}, which solve the multi-modal sensor learning problem as an MOO problem.

\textbf{Balanced multi-modal sensor learning for better generalization.} We first provide qualitative evidence that MIMO improves generalization in a real-world multi-modal sensor-fusion task using the CREMA-D dataset. Figure \ref{fig:cremad-illustration} (Left) shows the learning behavior of vanilla MSL and MIMO. While MIMO learns the slow-to-learn visual modality, vanilla MSL overfits the audio modality, resulting in poor test performance. 
Furthermore, in Figure \ref{fig:cremad-illustration} (Middle and Right), we investigate the loss landscape around the models trained using vanilla MSL and the proposed MIMO method in a reduced-dimensional space \citep{li2018visualizing}. Vanilla MSL achieves lower training loss (solid black contours) than MIMO, but exhibits poorer test loss (dashed yellow contours). In addition, the training-accuracy disparity between the audio and visual modalities (blue-green shading) is much larger in favor of the audio modality. In contrast, MIMO yields a more balanced training accuracy between the audio and visual modalities and a better test-loss profile than vanilla MSL, at the expense of slightly worse training loss. 
This suggests that balanced multi-modal sensor learning prevents the model from overfitting to a specific modality, thereby improving generalization.

\textbf{Comparison with baselines.} Next, we demonstrate the performance gains of MIMO on real-world multi-modal benchmarks relative to existing balanced multi-modal sensor learning methods and MOO methods. Table \ref{tab:cremad-urfunny} reports the performance of MIMO and several fusion and MOO baselines on the CREMA-D and UR-Funny classification benchmarks. On CREMA-D, MIMO achieves the best test accuracy among the baselines. Comparing against the individual-modality results also shows that MIMO attains superior performance through balanced multi-modal sensor learning, whereas vanilla MSL fails to match the best-performing individual modality (the visual modality). Furthermore, naively applying MOO methods to the fused and uni-modal sensor objectives (e.g., EW and MGDA) does not improve fused performance, because such approaches provide no fine-grained control over the trade-off between uni-modal sensor and fused-objective optimization. For example, MGDA is heavily biased toward the audio modality, leading to poor fused accuracy.

\begin{table}
\setlength{\tabcolsep}{0.1em}
\caption{Comparison using VGGSound dataset.}
\label{tab:vggsound}
{\small
\begin{tabular}{lcccc}
\hline
& Acc {\tiny(\%)} & Acc$_a$ {\tiny(\%)}& Acc$_v$ {\tiny(\%)} & $t(s)$\\
\hline
\textbf{ MSL}  & $60.8${\tiny $\pm 0.13$} & $42.83${\tiny $\pm 1.04$} & $15.43${\tiny $\pm 1.18$}  & $0.011${\tiny $\pm 0.001$}\\
\textbf{OGM-GE}  & $62.13${\tiny $\pm 1.31$} & $32.50${\tiny $\pm 2.26$} & $22.00${\tiny $\pm0.01$}  & $0.121${\tiny $\pm 0.007$}\\
\textbf{EW}     & $63.90${\tiny $\pm 1.58$} & $\bm{48.77}${\tiny $\pm 2.01$} & $25.20${\tiny $\pm 1.40$}  & $\underline{0.006}${\tiny $\pm 0.001$}\\
\textbf{MMPareto}&$66.07${\tiny $\pm 1.04$} & $48.07${\tiny $\pm 1.37$} & $28.87${\tiny $\pm 1.44$}  & $0.389${\tiny $\pm 0.053$}\\
\textbf{MIMO}     &  $\bm{69.10}${\tiny $\pm 1.13$} & $41.47${\tiny $\pm 1.04$} & $\bm{38.20}${\tiny $\pm 1.01$} & $0.019${\tiny $\pm 0.004$}\\
\bottomrule
\end{tabular}}
\end{table}

\begin{table}[t]
\caption{Comparison using Kinetics-Sound dataset.}
\label{tab:Kinetics-Sound}
{\small
\setlength{\tabcolsep}{0.15em} 
\begin{tabular}{lcccc}
\hline
& Acc {\tiny(\%)} & Acc$_a$ {\tiny(\%)}& Acc$_v$ {\tiny(\%)} & $t(s)$\\
\hline
\textbf{ MSL}  & $59.83${\tiny $\pm 1.78$} & $41.2${\tiny $\pm 6.34$} & $18.67${\tiny $\pm 1.56$}  & $0.023${\tiny $\pm 0.007$}\\
\textbf{OGM-GE}  & $63.73${\tiny $\pm 1.37$} & $44.10${\tiny $\pm 0.01$} & $22.57${\tiny $\pm2.56$}  & $0.247${\tiny $\pm 0.119$}\\
\textbf{EW}     & $60.73${\tiny $\pm 1.77$} & $45.13${\tiny $\pm 2.56$} & $33.3${\tiny $\pm 2.34$}  & $0.026${\tiny $\pm 0.010$}\\
\textbf{MMPareto}&$68.60${\tiny $\pm 1.41$} & $\bm{48.07}${\tiny $\pm 1.37$} & $37.23${\tiny $\pm 1.37$}  & $0.689${\tiny $\pm 0.098$}\\
\textbf{MIMO}     &  $\bm{69.60}${\tiny $\pm 1.41$} & $45.07${\tiny $\pm 1.04$} & $\bm{43.13}${\tiny $\pm 1.78$} & \underline{0.039}{\tiny $\pm$ 0.015}\\
\bottomrule
\end{tabular}}
\vspace{-0.2cm}
\end{table}
 
We then compare MIMO against MSL and MOO baselines on the three-modality UR-Funny benchmark and observe that MIMO performs comparably to or better than the baselines, while maintaining a subroutine time close to that of vanilla MSL; this is consistent with the observations on CREMA-D. Next, we evaluate performance on the Kinetics-Sound benchmark. The results in Table \ref{tab:Kinetics-Sound} show that MIMO outperforms both MOO baselines and balanced multi-modal sensor learning baselines. We attribute this superior performance to more balanced learning across modalities, as evidenced by the smaller disparity in uni-modal sensor accuracies for audio and visual modalities (only $\sim 3\%$ for MIMO, whereas the next-smallest disparity, achieved by MMPareto, is $\sim 11\%$). Similar observations can be made from the VGGSound results in Table \ref{tab:vggsound}.

\begin{figure}[t]
    \begin{center}
        \includegraphics[width=0.9\linewidth]{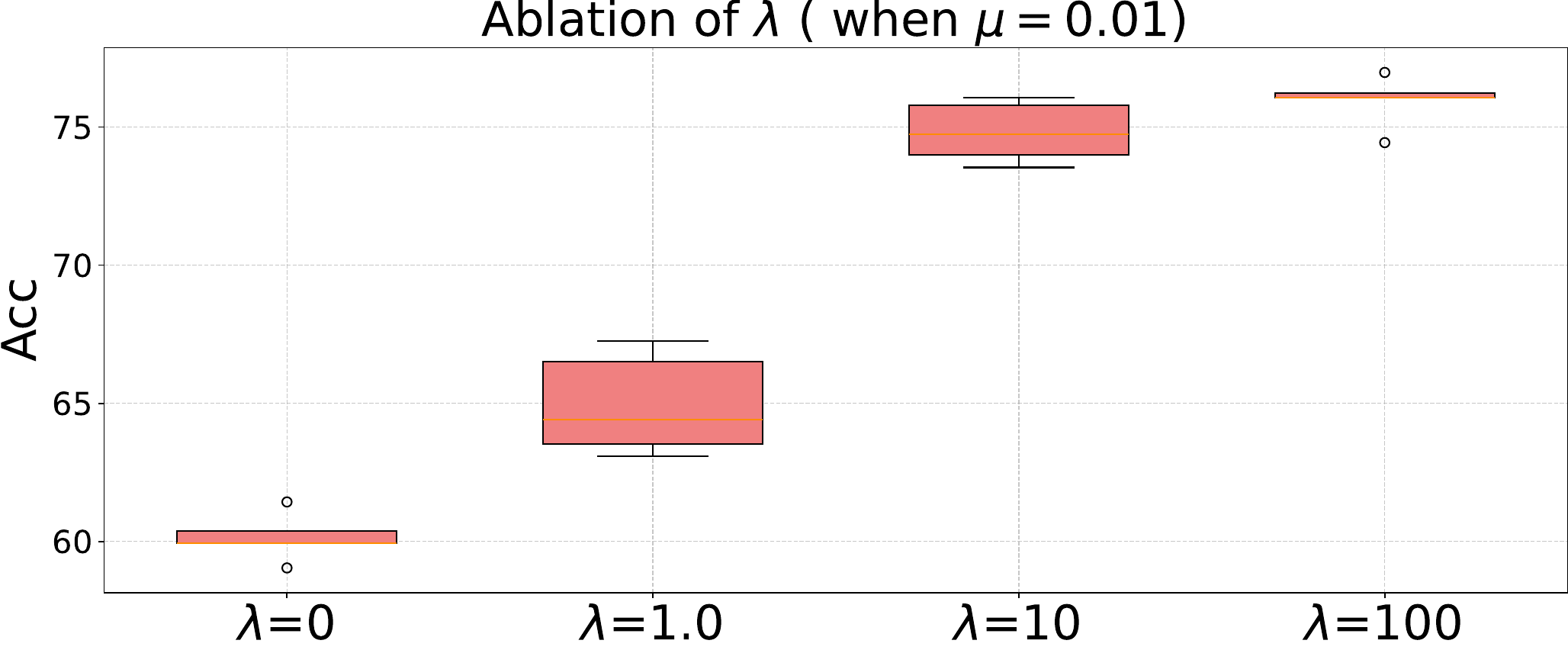}\\
        \vspace{0.2cm}
        \includegraphics[width=0.9\linewidth]{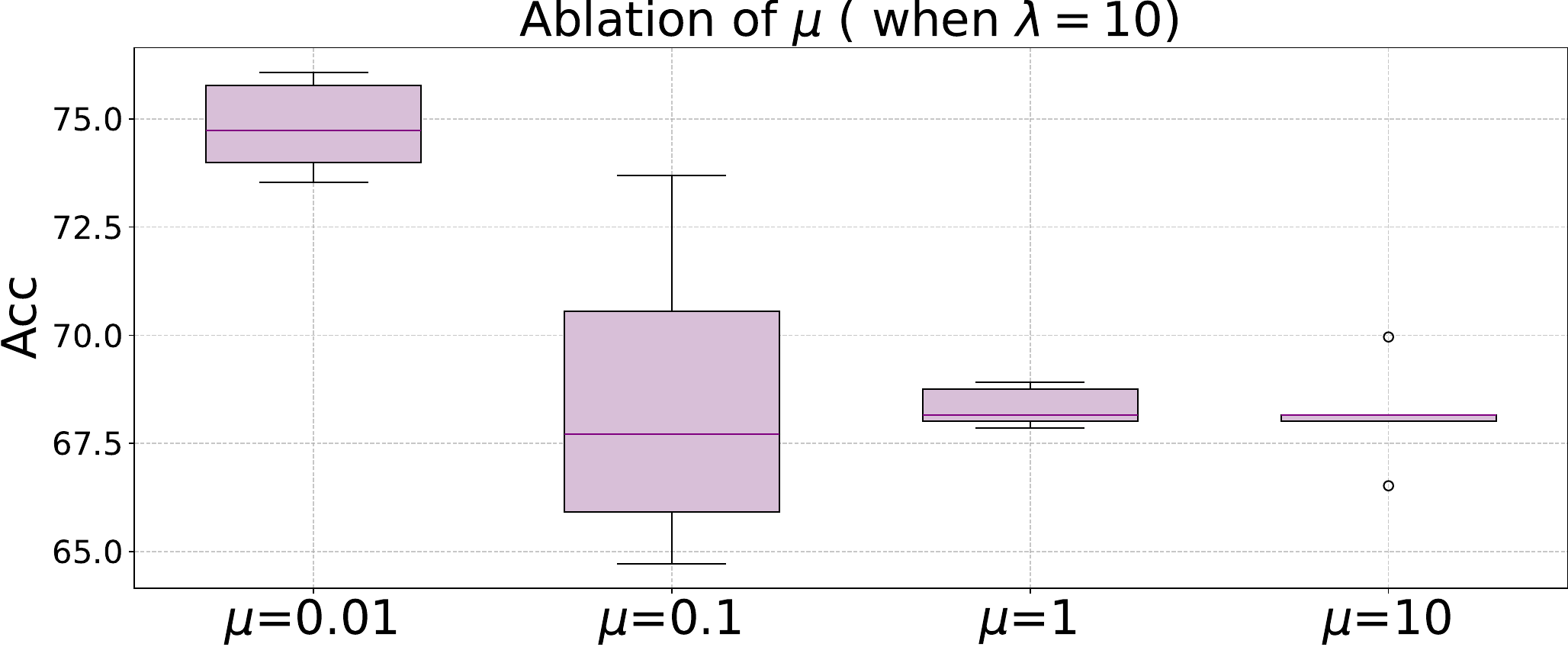}
        \centering
        \caption{Ablation of hyperparameters.}
        \label{fig:param-ablation}
    \end{center}
\end{figure}
 
 Furthermore, in most of the experiments above, MIMO has the fastest subroutine time (underlined values in the $t(s)$ column) among the top three best-performing methods, and its runtime is on the same order as vanilla MSL. In the VGGSound experiments, MIMO achieves a speed-up of up to $\sim 20\times$ relative to the next-best-performing baseline (MMPareto). While MIMO consistently outperforms the baselines across datasets, simple methods such as vanilla MSL can still work reasonably well on some datasets, such as UR-FUNNY. This may be due to the nature of the dataset, which does not satisfy the conditions for imbalance in MSL \citep{lu2024theory} and therefore does not exhibit severe modality imbalance. However, since MIMO does not incur significant computational overhead, applying it in such cases is reasonable.

 \textbf{Ablation of MIMO parameters.} In Figure \ref{fig:param-ablation}, we provide an ablation study of the choice of $\lambda$ and $\mu$ for MIMO on the CREMA-D dataset. We observe that, for very small $\lambda$ ($= 1.0$), performance is worse than for larger $\lambda$ values ($\geq 10$), which is expected because the effect of the constraint is weaker when $\lambda$ is small. We also observe that increasing $\lambda$ can improve performance, although the gains become marginal for larger $\lambda$ values ($\geq 10$). For $\mu$, smaller values lead to good performance, whereas larger values significantly degrade performance. This is also expected: when $\mu$ is large, the smoothed max deviates too much from the max function, preventing MIMO from prioritizing the worst-performing uni-modal sensor objective.

\section{Conclusions and Future Work}\label{sec:conclusion}

In this paper, we proposed a new formulation for balanced multi-modal sensor learning that prioritizes the worst-performing uni-modal sensor objective while optimizing the fused multi-modal sensor objective. This perspective improves modality balance and supports robust multi-modal sensor representations relevant to learning-enabled control pipelines. We then proposed MIMO, a simple gradient-based algorithm with convergence guarantees and without computationally expensive inner subroutines. Empirical evaluations show that MIMO outperforms existing balanced multi-modal sensor learning methods while substantially reducing computation time. This work focuses on late-fusion architectures with concatenation/sum fusion; extending MIMO-style MOO formulations to early/hybrid fusion and to explicit closed-loop control criteria is a natural next step.

\bibliography{myabrv,mtl, moo, mml, bo, PL}

@string{NIPS= "Proc. Advances in Neural Information Processing Systems"}

@string{ICML= "Proc. International Conference on Machine Learning"}

@string{ICLR= "Proc. International Conference on Learning Representations"}

@string{ICASSP= "Proc. IEEE International Conference on Acoustics, Speech and Signal Processing"}

@string{NIPS2023_loc= "New Orleans, LA"}

@string{ICML2018_loc= "Stockholm, Sweden"}

@string{ICLR2023_loc= "Kigali, Rwanda"}

@string{ICLR2024_loc= "Vienna, Austria"}

@inproceedings{xiao2023generalized,
  title={A Generalized Alternating Method for Bilevel Optimization under the Polyak-{\L}ojasiewicz Condition},
  author={Xiao, Quan and Lu, Songtao and Chen, Tianyi},
  booktitle=NIPS,
  year={2023},
  address=NIPS2023_loc
}

@inproceedings{kwon2023penalty,
  title={On Penalty Methods for Nonconvex Bilevel Optimization and First-Order Stochastic Approximation},
  author={Kwon, Jeongyeol and Kwon, Dohyun and Wright, Steve and Nowak, Robert},
  booktitle=ICLR,
  year={2024},
  address=ICLR2024_loc
}

@inproceedings{chen2024finding,
  title={On finding small hyper-gradients in bilevel optimization: Hardness results and improved analysis},
  author={Chen, Lesi and Xu, Jing and Zhang, Jingzhao},
  booktitle=COLT,
  pages={947--980},
  year={2024},
  organization={PMLR}
}

@inproceedings{fazel2018global,
  title={Global convergence of policy gradient methods for the linear quadratic regulator},
  author={Fazel, Maryam and Ge, Rong and Kakade, Sham and Mesbahi, Mehran},
  booktitle=ICML,
  year={2018},
  address=ICML2018_loc
}

@article{sun2018geometric,
  title={A geometric analysis of phase retrieval},
  author={Sun, Ju and Qu, Qing and Wright, John},
  journal={Foundations of Computational Mathematics},
  volume={18},
  pages={1131--1198},
  year={2018},
  publisher={Springer}
}

@inproceedings{marion2023implicit,
  title={Implicit regularization of deep residual networks towards neural ODEs},
  author={Marion, Pierre and Wu, Yu-Han and Sander, Michael E and Biau, G{\'e}rard},
  booktitle=NIPS,
  year={2023},
  address=NIPS2023_loc
}

@inproceedings{chenover,
  title={Over-parameterized Model Optimization with {Polyak}-Lojasiewicz Condition},
  author={Chen, Yixuan and Shi, Yubin and Dong, Mingzhi and Yang, Xiaochen and Li, Dongsheng and Wang, Yujiang and Dick, Robert P and Lv, Qin and Zhao, Yingying and Yang, Fan and others},
  booktitle=ICLR,
  year=2023,
  address=ICLR2023_loc
}

@book{rockafellar2009variational,
  title={Variational analysis},
  author={Rockafellar, R Tyrrell and Wets, Roger J-B},
  volume={317},
  year={2009},
  publisher={Springer Science \& Business Media}
}

@book{ulbrich2011semismooth,
  title={Semismooth Newton methods for variational inequalities and constrained optimization problems in function spaces},
  author={Ulbrich, Michael},
  year={2011},
  publisher={SIAM}
}

@inproceedings{shen2023penalty,
  title={On penalty-based bilevel gradient descent method},
  author={Shen, Han and Chen, Tianyi},
  booktitle={International Conference on Machine Learning},
  pages={30992--31015},
  year={2023},
  organization={PMLR}
}

@article{wang2022tag,
  title={Tag: Boosting text-vqa via text-aware visual question-answer generation},
  author={Wang, Jun and Gao, Mingfei and Hu, Yuqian and Selvaraju, Ramprasaath R and Ramaiah, Chetan and Xu, Ran and JaJa, Joseph F and Davis, Larry S},
  journal={arXiv preprint arXiv:2208.01813},
  year={2022}
}

@article{shridhar2020alfworld,
  title={Alfworld: Aligning text and embodied environments for interactive learning},
  author={Shridhar, Mohit and Yuan, Xingdi and C{\^o}t{\'e}, Marc-Alexandre and Bisk, Yonatan and Trischler, Adam and Hausknecht, Matthew},
  journal={arXiv preprint arXiv:2010.03768},
  year={2020}
}

@inproceedings{zhang2019neural,
  title={Neural machine translation with universal visual representation},
  author={Zhang, Zhuosheng and Chen, Kehai and Wang, Rui and Utiyama, Masao and Sumita, Eiichiro and Li, Zuchao and Zhao, Hai},
  booktitle={International Conference on Learning Representations},
  year={2019}
}

@article{reed2022generalist,
  title={A generalist agent},
  author={Reed, Scott and Zolna, Konrad and Parisotto, Emilio and Colmenarejo, Sergio Gomez and Novikov, Alexander and Barth-Maron, Gabriel and Gimenez, Mai and Sulsky, Yury and Kay, Jackie and Springenberg, Jost Tobias and others},
  journal={arXiv preprint arXiv:2205.06175},
  year={2022}
}

@inproceedings{vielzeuf2018centralnet,
  title={Centralnet: a multilayer approach for multimodal fusion},
  author={Vielzeuf, Valentin and Lechervy, Alexis and Pateux, St{\'e}phane and Jurie, Fr{\'e}d{\'e}ric},
  booktitle={Proceedings of the European Conference on Computer Vision (ECCV) Workshops},
  pages={0--0},
  year={2018}
}

@article{cao2014crema,
  title={Crema-d: Crowd-sourced emotional multimodal actors dataset},
  author={Cao, Houwei and Cooper, David G and Keutmann, Michael K and Gur, Ruben C and Nenkova, Ani and Verma, Ragini},
  journal={IEEE transactions on affective computing},
  volume={5},
  number={4},
  pages={377--390},
  year={2014},
  publisher={IEEE}
}

@inproceedings{ma2022multimodal,
  title={Are multimodal transformers robust to missing modality?},
  author={Ma, Mengmeng and Ren, Jian and Zhao, Long and Testuggine, Davide and Peng, Xi},
  booktitle={Proceedings of the IEEE/CVF Conference on Computer Vision and Pattern Recognition},
  pages={18177--18186},
  year={2022}
}

@article{hu2022shape,
  title={SHAPE: An Unified Approach to Evaluate the Contribution and Cooperation of Individual Modalities},
  author={Hu, Pengbo and Li, Xingyu and Zhou, Yi},
  journal={arXiv preprint arXiv:2205.00302},
  year={2022}
}

@article{allen2020towards,
  title={Towards understanding ensemble, knowledge distillation and self-distillation in deep learning},
  author={Allen-Zhu, Zeyuan and Li, Yuanzhi},
  journal={arXiv preprint arXiv:2012.09816},
  year={2020}
}

@article{han2022trusted,
  title={Trusted multi-view classification with dynamic evidential fusion},
  author={Han, Zongbo and Zhang, Changqing and Fu, Huazhu and Zhou, Joey Tianyi},
  journal={IEEE transactions on pattern analysis and machine intelligence},
  volume={45},
  number={2},
  pages={2551--2566},
  year={2022},
  publisher={IEEE}
}

@inproceedings{zadeh2018multimodal,
  title={Multimodal language analysis in the wild: Cmu-mosei dataset and interpretable dynamic fusion graph},
  author={Zadeh, AmirAli Bagher and Liang, Paul Pu and Poria, Soujanya and Cambria, Erik and Morency, Louis-Philippe},
  booktitle={Proceedings of the 56th Annual Meeting of the Association for Computational Linguistics (Volume 1: Long Papers)},
  pages={2236--2246},
  year={2018}
}

@inproceedings{tian2018audio,
  title={Audio-visual event localization in unconstrained videos},
  author={Tian, Yapeng and Shi, Jing and Li, Bochen and Duan, Zhiyao and Xu, Chenliang},
  booktitle={Proceedings of the European conference on computer vision (ECCV)},
  pages={247--263},
  year={2018}
}

@inproceedings{wang2020makes,
  title={What makes training multi-modal classification networks hard?},
  author={Wang, Weiyao and Tran, Du and Feiszli, Matt},
  booktitle={Proceedings of the IEEE/CVF conference on computer vision and pattern recognition},
  pages={12695--12705},
  year={2020}
}

@inproceedings{huang2022modality,
  title={Modality competition: What makes joint training of multi-modal network fail in deep learning?(provably)},
  author={Huang, Yu and Lin, Junyang and Zhou, Chang and Yang, Hongxia and Huang, Longbo},
  booktitle={International Conference on Machine Learning},
  pages={9226--9259},
  year={2022},
  organization={PMLR}
}

@inproceedings{fujimori2020modality,
  title={Modality-specific learning rate control for multimodal classification},
  author={Fujimori, Naotsuna and Endo, Rei and Kawai, Yoshihiko and Mochizuki, Takahiro},
  booktitle={Pattern Recognition: 5th Asian Conference, ACPR 2019, Auckland, New Zealand, November 26--29, 2019, Revised Selected Papers, Part II 5},
  pages={412--422},
  year={2020},
  organization={Springer}
}

@inproceedings{yao2022modality,
  title={Modality-specific learning rates for effective multimodal additive late-fusion},
  author={Yao, Yiqun and Mihalcea, Rada},
  booktitle={Findings of the Association for Computational Linguistics: ACL 2022},
  pages={1824--1834},
  year={2022}
}

@inproceedings{peng2022balanced,
  title={Balanced multimodal learning via on-the-fly gradient modulation},
  author={Peng, Xiaokang and Wei, Yake and Deng, Andong and Wang, Dong and Hu, Di},
  booktitle={Proceedings of the IEEE/CVF Conference on Computer Vision and Pattern Recognition},
  pages={8238--8247},
  year={2022}
}

@inproceedings{li2023boosting,
  title={Boosting Multi-modal Model Performance with Adaptive Gradient Modulation},
  author={Li, Hong and Li, Xingyu and Hu, Pengbo and Lei, Yinuo and Li, Chunxiao and Zhou, Yi},
  booktitle={Proceedings of the IEEE/CVF International Conference on Computer Vision},
  pages={22214--22224},
  year={2023}
}

@article{hasan2019ur,
  title={UR-FUNNY: A multimodal language dataset for understanding humor},
  author={Hasan, Md Kamrul and Rahman, Wasifur and Zadeh, Amir and Zhong, Jianyuan and Tanveer, Md Iftekhar and Morency, Louis-Philippe and others},
  journal={arXiv preprint arXiv:1904.06618},
  year={2019}
}

@inproceedings{arandjelovic2017look,
  title={Look, listen and learn},
  author={Arandjelovic, Relja and Zisserman, Andrew},
  booktitle={Proceedings of the IEEE international conference on computer vision},
  pages={609--617},
  year={2017}
}

@article{kay2017kinetics,
  title={The kinetics human action video dataset},
  author={Kay, Will and Carreira, Joao and Simonyan, Karen and Zhang, Brian and Hillier, Chloe and Vijayanarasimhan, Sudheendra and Viola, Fabio and Green, Tim and Back, Trevor and Natsev, Paul and others},
  journal={arXiv preprint arXiv:1705.06950},
  year={2017}
}

@inproceedings{chen2020vggsound,
  title={Vggsound: A large-scale audio-visual dataset},
  author={Chen, Honglie and Xie, Weidi and Vedaldi, Andrea and Zisserman, Andrew},
  booktitle={ICASSP 2020-2020 IEEE International Conference on Acoustics, Speech and Signal Processing (ICASSP)},
  pages={721--725},
  year={2020},
  organization={IEEE}
}

@article{zhang2023theory,
  title={A Theory of Unimodal Bias in Multimodal Learning},
  author={Zhang, Yedi and Latham, Peter E and Saxe, Andrew},
  journal={arXiv preprint arXiv:2312.00935},
  year={2023}
}

@article{lu2024theory,
  title={A Theory of Multimodal Learning},
  author={Lu, Zhou},
  journal={Advances in Neural Information Processing Systems},
  volume={36},
  year={2024}
}

@inproceedings{wei2024innocent,
  title={MMPareto: boosting multimodal learning with innocent unimodal assistance},
  author={Wei, Yake and Hu, Di},
  booktitle={International Conference on Machine Learning},
  year={2024}
}

@article{Desideri2012mgda,
  title={{Multiple-gradient Descent Algorithm (MGDA) for Multi-objective Optimization}},
  author={Désidéri, Jean-Antoine},
  journal={Comptes Rendus Mathematique},
  volume={350},
  number={5-6},
  page={313-318},
  year={2012}
}

@article{fliege2019complexity,
  title={{Complexity of Gradient Descent for Multi-objective Optimization}},
  author={Fliege, J{\"o}rg and Vaz, A Ismael F and Vicente, Lu{\'\i}s Nunes},
  journal={Optimization Methods and Software},
  volume={34},
  number={5},
  pages={949--959},
  year={2019},
  publisher={Taylor \& Francis}
}

@book{nesterov2018lectures,
  title={Lectures on convex optimization},
  author={Nesterov, Yurii},
  volume={137},
  year={2018},
  publisher={Springer}
}

@book{miettinen1999nonlinear,
  title={Nonlinear multiobjective optimization},
  author={Miettinen, Kaisa},
  volume={12},
  year={1999},
  publisher={Springer Science \& Business Media}
}

@article{lin2024smooth,
  title={Smooth Tchebycheff Scalarization for Multi-Objective Optimization},
  author={Lin, Xi and Zhang, Xiaoyuan and Yang, Zhiyuan and Liu, Fei and Wang, Zhenkun and Zhang, Qingfu},
  journal={arXiv preprint arXiv:2402.19078},
  year={2024}
}

@article{nesterov2005smooth,
  title={Smooth minimization of non-smooth functions},
  author={Nesterov, Yu},
  journal={Mathematical programming},
  volume={103},
  pages={127--152},
  year={2005},
  publisher={Springer}
}

@article{li2018visualizing,
  title={Visualizing the loss landscape of neural nets},
  author={Li, Hao and Xu, Zheng and Taylor, Gavin and Studer, Christoph and Goldstein, Tom},
  journal={Advances in neural information processing systems},
  volume={31},
  year={2018}
}

@article{liu2022loss,
  title={Loss landscapes and optimization in over-parameterized non-linear systems and neural networks},
  author={Liu, Chaoyue and Zhu, Libin and Belkin, Mikhail},
  journal={Applied and Computational Harmonic Analysis},
  volume={59},
  pages={85--116},
  year={2022},
  publisher={Elsevier}
}

@string{ICLR = "Proc. of International Conference on Learning Representations"}

@string{ICML = "Proc. of International Conference on Machine Learning"}

@string{COLT = "Proc. of Conference on Learning Theory"}

@string{arxiv = "arXiv preprint: "}
\bibliographystyle{IEEEtran}

\clearpage
\onecolumn
\appendices

\begin{center}
{\large \bf Supplementary Material for
``Balancing Multi-modal Sensor Learning via Multi-objective Optimization"}
\end{center}

\section{Notations}

A summary of notations used in this work is listed in Table~\ref{tab:notations} for ease of reference.
\begin{table}[ht]
\caption{Notations and their descriptions.}
  \label{tab:notations}
  \footnotesize
  \centering
  \begin{tabular}{l|l l }
  \toprule
  \textbf{Notation} & \textbf{Description}   \\
  \midrule
  $K$ & Number of modalities considered. We use K=2 in most parts of the paper for conciseness \\
  $N$ & Number of datapoints in the multi-modal sensor dataset $\mathcal{D}_{mm}$ \\
  $k$ & Index used to denote modality, $k\in[K]$ \\
  $i$ & Index used to denote datapoint, $i\in[N]$ \\
  \midrule
  \multicolumn{2}{c}{Dataset}\\
  \midrule
  $x_i^{(m_k)}$ & Input corresponding to modality $m_k$ for the datapoint index $i$ \\
  $y_i$ & Target output for the datapoint index $i$ \\
  $\hat{y}_i$ & Multi-modal output of the vanilla MSL/MIMO model for the datapoint index $i$ \\
  $\hat{y}_i^{(m_k)}$ & Uni-modal output of the MIMO model for the modality $m_k$ for the datapoint index $i$ \\
  \midrule
  \multicolumn{2}{c}{Model}\\
  \midrule
  $\vartheta_{mm}$ & Multi-modal head in the vanilla MSL/MIMO model (see Figure \ref{fig:overview}) \\
  $\vartheta_{m_k}$ & Uni-modal head for the modality $m_k$ in the MIMO model (see Figure \ref{fig:overview}) \\
  $\theta_{m_k}$ & Uni-modal encoder for the modality $m_k$ in the vanilla MSL/MIMO model (see Figure \ref{fig:overview}) \\
  $\boldsymbol{\Theta_{mm}}$ & \makecell[l]{Concatenation of all the components in the vanilla MSL model,\\ i.e. $\boldsymbol{\Theta_{mm}}:= [\vartheta_{mm}; \theta_{m_1}; \theta_{m_2}; \dots ; \theta_{m_K}]$}\\
  $\boldsymbol{\hat{\Theta}_{mm}}$ & \makecell[l]{Concatenation of the MIMO model parameters,\\ i.e. $\boldsymbol{\hat{\Theta}_{mm}}:= [\vartheta_{mm}; \vartheta_{m_1}; \vartheta_{m_2}; \dots ; \vartheta_{m_K}; \theta_{m_1}; \theta_{m_2}; \dots ; \theta_{m_K}]$}\\
  $\boldsymbol{\Theta_{m_k}}$ & \makecell[l]{Concatenation of all the $m_k$ modality specific components in the MIMO model,\\ i.e. $\boldsymbol{\Theta_{m_k}}:= [\vartheta_{m_k}; \theta_{m_k}]$}\\
  \midrule
  \multicolumn{2}{c}{Objectives}\\
  \midrule
  $f_{mm}(\boldsymbol{\Theta_{mm}})$ & Vanilla multi-modal sensor fusion objective\\
  $f_{m_k}(\boldsymbol{\Theta_{m_k}})$ & Uni-modal objective for modality $m_k$ induced by uni-modal sensor head $\vartheta_{m_k}$ in MIMO model\\
  $\hat{f}_{mm}(\boldsymbol{\hat{\Theta}_{mm}})$ & MIMO objective, which is a combination of $f_{mm}$ and $f_{m_k}$ for all $k\in[K]$ (defined in \eqref{eq:penalty-smooth-cheb-two} for K=2)\\
  \midrule
  \multicolumn{2}{c}{Toy Illustration}\\
  \midrule
  $\vartheta_{mm, m_k}$ & Parameter matrix for the partition of $\vartheta_{mm}$ that corresponds to the modality $m_k$ \\
  $C_{m_k}$ &  Empirical input correlation matrix for modality $m_k$ (defined in Section \ref{app:toy-illustration})\\
  $C_{m_km_{3-k}}$ & Empirical cross-correlation matrix between modality  $m_k$ and modality $m_{3-k}$ (defined in Section \ref{app:toy-illustration})\\
  $C_{ym_k}$ & Empirical input-output correlation matrix for modality $m_k$ (defined in Section \ref{app:toy-illustration})\\
  \bottomrule
  \end{tabular}
  \vspace{0.2cm}
\end{table}

Let the empirical input and input-output correlation matrices for modality $m_1$  be $C_{m_1}$ and $C_{ym_1}$ (similarly for modality $m_2$). Also let the cross-correlation matrices between $m_1$ and $m_2$ be $C_{m_1m_2}$ and $C_{m_2m_1}$

\section{Details of Toy Example} \label{app:toy-illustration}

In this section, we provide the implementation details of the toy experiment used to generate the learning curves given in Figure \ref{fig:overview}. This experiment is motivated by a similar illustration given in \citep{zhang2023theory}.

\paragraph{Datset.} To generate multi-modal data $\mathcal{D}_{mm} := \{ x^{(m_1)}_i, x^{(m_2)}_i , y_i\}_{i=1}^N$, we sample each element of $x^{(m_1)}_i$ from $\mathcal{N}(0, 25)$ and each element of $x^{(m_2)}_i$ from $\mathcal{N}(0, 0.25)$, where $x^{(m_1)}_i,x^{(m_2)}_i \in \mathbb{R}^{50}$, $\mathcal{N}(\mu, \sigma^2)$ are Gaussian distributions with mean $\mu$ and variance $\sigma^2$. We set the number of data points $N=700$. The label for each datapoint is generated as $y_i = 0.001x^{(m_1)}_i + x^{(m_2)}_i$ . The dataset is generated in this way so that it satisfies the condition for superficial modality preference given by
\begin{equation}
    \Vert C_{ym_1} \Vert > \Vert C_{ym_2} \Vert \quad \text{and} \quad C_{ym_1}C_{m_1}^{-1}C_{ym_1}^\top < C_{ym_2}C_{m_2}^{-1}C_{ym_2}^\top,
\end{equation}
where $C_{m_k}:=\frac{1}{N}\Sigma_{i=1}^N x^{(m_k)}_i (x^{(m_k)}_i)^\top$, $C_{m_km_{3-k}}:=\frac{1}{N}\Sigma_{i=1}^N x^{(m_k)}_i (x^{(m_{3-k})}_i)^\top$, and $C_{ym_k}:=\frac{1}{N}\Sigma_{i=1}^N y_i (x^{(m_k)}_i)^\top$.
The derivation of the condition (16) for superficial modality preference follow the derivation steps given in \citep{zhang2023theory} Appendix F.
\paragraph{Models.}For the vanilla multi-modal sensor fusion model (as shown in Figure \ref{fig:overview} (a)) we use a linear layer of input size $50$ and output size $100$ as the modality encoder for each modality $m_1$ and $m_2$, and then use linear layers of input size $100$ and output size $1$ for the modality-specific fusion head part for each modality $m_1$ and $m_2$. Finally, the multi-modal output is obtained by summing the output of each modality-specific part. For the MIMO model, we use the same architecture for the multi-modal part of the model, and use two additional linear layers, each with input size $100$ and output size $1$ to generate uni-modal sensor outputs (as shown in Figure \ref{fig:overview} (b)).

\paragraph{Optimization.} We use a learning rate of $0.01$ for both vanilla MSL and MIMO methods. For the MIMO method, we set $\lambda=10$ and $\mu=0.2$. The expressions of the gradients used in vanilla MSL and MIMO are summarized in Table \ref{tab:gradients}. The derivation of gradients follow a similar approach to that given in \citep{zhang2023theory} Appendix A.

\paragraph{Superficial modality preference.} In Figure \ref{fig:overview} (a), we can see that vanilla MSL is quick to learn modality $m_1$. However, it does not contribute to minimizing the multi-modal sensor fusion objective compared to modality $m_2$. This phenomenon is known as ``superficial modality preference'' \citep{zhang2023theory}, and occurs due to the properties of the dataset. More concretely, let the time taken to reach $\mathcal{M}_{m_1}$ and $\mathcal{M}_{m_2}$ are $t_{m_1}$ and $t_{m_2}$, respectively. Furthermore, let the objective value at the manifolds $\mathcal{M}_{m_1}$ and $\mathcal{M}_{m_2}$ be $f_{mm}\left(\mathcal{M}_{m_1}\right)$ and $f_{mm}\left(\mathcal{M}_{m_2}\right)$, respectively. Then we say model has ``superficial modality preference'' \citep{zhang2023theory} if the following condition holds
\begin{equation}
    t_{m_1} < t_{m_2} \quad \text{and} \quad f_{mm}\left(\mathcal{M}_{m_1}\right) > f_{mm}\left(\mathcal{M}_{m_2}\right).
\end{equation}
It can be shown that if the dataset statistics satisfy the following condition, applying SGD on $f_{mm}$ with model $\boldsymbol{\Theta}$ will result in superficial modality preference:
\begin{equation}
    \Vert C_{ym_1} \Vert > \Vert C_{ym_2} \Vert \quad \text{and} \quad C_{ym_1}C_{m_1}^{-1}C_{ym_1}^\top < C_{ym_2}C_{m_2}^{-1}C_{ym_2}^\top.
\end{equation}
Note that the condition depends only on the statistics of each modality data. Thus, applying SGD on multi-modal sensor fusion objective $f_{mm}$ parameterized by a late fusion multi-modal $\boldsymbol{\Theta}$ can result in the model giving priority to one modality, which may not be contributing most in minimizing the objective. The data set used in this toy example is generated in such a way that the above conditions for superficial modality preference are met.

\begin{table}[t]
\centering
{
    \centering
        \caption{Gradient of the objective for each layer of the network for vanilla MSL (column ${f}_{mm}({\boldsymbol{\Theta}})$) and MIMO (column ${f}_{mm}({\boldsymbol{\Theta}})$ + column $\Delta$). Column $\Delta$ contain the additional gradient components for $\hat{f}_{mm}(\boldsymbol{\hat{\Theta}_{mm}})$ (if any) compared to that of $f_{mm}(\boldsymbol{\Theta})$. Rows  $\nabla_{ \vartheta_{m_1}}$ and $\nabla_{ \vartheta_{m_2}}$ for column $f_{mm}(\boldsymbol{\Theta})$ is empty because only MIMO model $\boldsymbol{\hat{\Theta}_{mm}}$ contains $\vartheta_{m_1}$ and $\vartheta_{m_2}$. In column $\Delta$, $\lambda_i = \lambda$ if $i = \arg\!\max\limits_{{\scriptscriptstyle i\in\{1,2\}}} \left( f_{m_k}(\vartheta_{m_k}, \theta_{m_k}) - f^*_{m_k} \right)$, else $0$, where $\lambda$ is the penalty parameter in \eqref{eq:mimo}.}\label{tab:gradients}
    \begin{tabular}{|c|c|c|}
        \hline
        & $f_{mm}(\boldsymbol{\Theta})$ & $\Delta = \hat{f}_{mm}(\boldsymbol{\hat{\Theta}_{mm}}) - f_{mm}(\boldsymbol{\Theta})$ \\
        \hline\hline
        $\nabla_{ \theta_{m_1}}$ & ${\vartheta_{mm, m_1}^\top} \left(C_{ym_1} - \vartheta_{mm, m_1}\theta_{m_1}C_{m_1} - \vartheta_{mm, m_2}\theta_{m_2}C_{m_2m_1} \right)$  & $ \lambda_1 \vartheta^{\top}_{m_1} \left(C_{ym_1} - \vartheta_{m_1}\theta_{m_1}C_{m_1} \right)$\\
        $\nabla_{ \theta_{m_2}}$ & ${\vartheta_{mm, m_2}^\top} \left(C_{ym_2} - \vartheta_{mm, m_1}\theta_{m_1}C_{m_1m_2} - \vartheta_{mm, m_2}\theta_{m_2}C_{m_2} \right)$  & $ \lambda_2 {\vartheta^{\top}_{m_2}} \left(C_{ym_2} - \vartheta_{m_2}\theta_{m_2}C_{m_2} \right)$ \\
        $\nabla_{ \vartheta_{mm, m_1}}$ & $\left(C_{ym_1} - \vartheta_{mm, m_1}\theta_{m_1}C_{m_1} - \vartheta_{mm, m_2}\theta_{m_2}C_{m_2m_1} \right){\theta_{m_1}^\top}$  & $-$ \\
        $\nabla_{ \vartheta_{mm, m_2}}$ & $\left(C_{ym_2} - \vartheta_{mm, m_1}\theta_{m_1}C_{m_1m_2} - \vartheta_{mm, m_2}\theta_{m_2}C_{m_2} \right){\theta_{m_2}^\top}$   & $-$ \\
        $\nabla_{ \vartheta_{m_1}}$ & $-$  &  $\lambda_1\left(C_{ym_1} - \vartheta_{m_1}\theta_{m_1}C_{m_1} \right){\theta_{m_1}^\top}$\\
        $\nabla_{ \vartheta_{m_2}}$ &  $-$ & $\lambda_2\left(C_{ym_2} - \vartheta_{m_2}\theta_{m_2}C_{m_2}\right){\theta_{m_2}^\top}$\\
        \hline
    \end{tabular}
    }
\end{table}

\section{Proof of Proposition \ref{prop:f-hat-smooth}.} \label{app:f-hat-smooth-proof}

In this section, we provide the proof for Proposition \ref{prop:f-hat-smooth}.
\begin{proof}
Consider any $\boldsymbol{\hat{\Theta}_{mm}} =  [\vartheta_{mm}; \vartheta_{m_1}; \vartheta_{m_2}; \theta_{m_1}; \theta_{m_2}]$, with $\boldsymbol{\Theta_{m_k}}=[\vartheta_{m_k}; \theta_{m_k}]$ for $k\in\{1, 2\}$. For brevity, we will omit the argument of the function/gradient in derivation; for example $\nabla_{\boldsymbol{\hat{\Theta}_{mm}}}f_{mm}(\boldsymbol{\hat{\Theta}_{mm}})$ will be denoted as $\nabla_{\boldsymbol{\hat{\Theta}_{mm}}}f_{mm}$. However, we will carefully consider the dependence of the function on the corresponding parameter, when we take gradients. Furthermore, we will denote the dimension of a vector parameter $v$ as $\text{dim}(v)$. Our goal in this proof is to show that $\nabla^2_{\boldsymbol{\hat{\Theta}_{mm}}}\hat{f}_{mm}\preceq \hat{L}_{mm}I_0$ for some $\hat{L}_{mm}>0$, where $\nabla^2_{\boldsymbol{\hat{\Theta}_{mm}}}\hat{f}_{mm}$ is the Hessian of $\hat{f}_{mm}$, and $I_0\in\R^{\text{dim}(\boldsymbol{\hat{\Theta}_{mm}})\times \text{dim}(\boldsymbol{\hat{\Theta}_{mm}})}$ is an identity matrix.
We first derive the gradient of $g_\mu := \mu\log\left(\sum_{k=1}^2\exp\left(\frac{f_{m_k} - f^*_{m_k}}{\mu} \right)\right)$ with repsect to $\boldsymbol{\hat{\Theta}_{mm}}$. We have
\begin{align}\label{eq:g-grad-0}
    \nabla_{\boldsymbol{\hat{\Theta}_{mm}}}g_\mu &= \frac{\mu}{\sum_{k=1}^2\exp\left(\frac{f_{m_k} - f^*_{m_k}}{\mu} \right)} \cdot \sum_{k=1}^2\frac{1}{\mu}\exp\left(\frac{f_{m_k} - f^*_{m_k}}{\mu} \right)\nabla_{\boldsymbol{\hat{\Theta}_{mm}}}f_{m_k} \nonumber\\
    &= \frac{1}{\sum_{k=1}^2\exp\left(\frac{f_{m_k} - f^*_{m_k}}{\mu} \right)}\sum_{k=1}^2\exp\left(\frac{f_{m_k} - f^*_{m_k}}{\mu} \right)\nabla_{\boldsymbol{\hat{\Theta}_{mm}}}f_{m_k}.
\end{align}
Then we can compute $\nabla^2_{\boldsymbol{\hat{\Theta}_{mm}}}g_\mu$ as
\begin{align}\label{eq:g-hessian-1}
    \nabla^2_{\boldsymbol{\hat{\Theta}_{mm}}}g_\mu &= \nabla_{\boldsymbol{\hat{\Theta}_{mm}}}\left(\frac{1}{\sum_{k=1}^2\exp\left(\frac{f_{m_k} - f^*_{m_k}}{\mu} \right)}\sum_{k=1}^2\exp\left(\frac{f_{m_k} - f^*_{m_k}}{\mu} \right)\nabla_{\boldsymbol{\hat{\Theta}_{mm}}}f_{m_k}\right) \nonumber\\
    &= \frac{\left( \sum_{k=1}^2\exp\left(\frac{f_{m_k} - f^*_{m_k}}{\mu} \right)\right) \nabla_{\boldsymbol{\hat{\Theta}_{mm}}}\left(\sum_{k=1}^2\exp\left(\frac{f_{m_k} - f^*_{m_k}}{\mu} \right)\nabla_{\boldsymbol{\hat{\Theta}_{mm}}}f_{m_k}\right)}{\left( \sum_{k=1}^2\exp\left(\frac{f_{m_k} - f^*_{m_k}}{\mu} \right)\right)^2} \nonumber\\
    &~~ - \frac{\left(\sum_{k=1}^2\exp\left(\frac{f_{m_k} - f^*_{m_k}}{\mu} \right)\nabla_{\boldsymbol{\hat{\Theta}_{mm}}}f_{m_k}\right)\nabla_{\boldsymbol{\hat{\Theta}_{mm}}}\left( \sum_{k=1}^2\exp\left(\frac{f_{m_k} - f^*_{m_k}}{\mu} \right)\right) }{\left( \sum_{k=1}^2\exp\left(\frac{f_{m_k} - f^*_{m_k}}{\mu} \right)\right)^2} \nonumber\\
    &= \sum_{k=1}^2\Psi\left( z_k \right) \left( \frac{1}{\mu}\nabla_{\boldsymbol{\hat{\Theta}_{mm}}}f_{m_k}\nabla_{\boldsymbol{\hat{\Theta}_{mm}}}f_{m_k}^\top + \nabla^2_{\boldsymbol{\hat{\Theta}_{mm}}}f_{m_k} \right) \nonumber\\&~~- \frac{1}{\mu}\left(\sum_{k=1}^2\Psi\left( z_k \right) \nabla_{\boldsymbol{\hat{\Theta}_{mm}}}f_{m_k} \right)\left(\sum_{k=1}^2\Psi\left( z_k \right) \nabla_{\boldsymbol{\hat{\Theta}_{mm}}}f_{m_k} \right)^\top,
\end{align}
where $z_k = \frac{f_{m_k} - f^*_{m_k}}{\mu}$, and $\Psi$ is the softmax operator given by $\Psi\left( z_i \right) = \frac{\exp(z_i)}{\sum_{k=1}^2 \exp(z_k)}$.
We then rewrite \eqref{eq:g-hessian-1} as
\begin{align}\label{eq:g-hessian-2}
    \nabla^2_{\boldsymbol{\hat{\Theta}_{mm}}}g_\mu =& \sum_{k=1}^2\Psi\left( z_k \right) \nabla^2_{\boldsymbol{\hat{\Theta}_{mm}}}f_{m_k} + \frac{1}{\mu}\Bigg[ \sum_{k=1}^2\Psi\left( z_k \right)\nabla_{\boldsymbol{\hat{\Theta}_{mm}}}f_{m_k}\nabla_{\boldsymbol{\hat{\Theta}_{mm}}}f_{m_k}^\top \nonumber\\
    &- \left(\sum_{k=1}^2\Psi\left( z_k \right) \nabla_{\boldsymbol{\hat{\Theta}_{mm}}}f_{m_k} \right)\left(\sum_{k=1}^2\Psi\left( z_k \right) \nabla_{\boldsymbol{\hat{\Theta}_{mm}}}f_{m_k} \right)^\top \Bigg].
\end{align}

Now, consider $\nabla^2_{\boldsymbol{\hat{\Theta}_{mm}}}f_{m_k}$ for $k\in\{1, 2\}$. Since $f_{m_k}$ is $L_{m_k}$-smooth (Assumption \ref{ass:smooth}), we have
\begin{align}\label{eq:g-hessian-2-1}
    &\nabla^2_{\boldsymbol{\Theta_{m_k}}}f_{m_k} \preceq L_{m_k} I_k, ~~\text{$I_k\in\R^{\text{dim}(\boldsymbol{\Theta_{m_k}})}$ is an identity matrix} \nonumber\\
    \implies& \nabla^2_{\boldsymbol{\hat{\Theta}_{mm}}}f_{m_k} \preceq L_{m_k} I_0 \nonumber\\
    \implies& v^\top(\nabla^2_{\boldsymbol{\hat{\Theta}_{mm}}}f_{m_k} - L_{m_k} I_0) v \leq 0 ~~ \text{for any }v\in\R^{\text{dim}(\boldsymbol{\hat{\Theta}_{mm}})} \nonumber\\
    \implies& \sum_{k=1}^2\Psi\left( z_k \right)v^\top(\nabla^2_{\boldsymbol{\hat{\Theta}_{mm}}}f_{m_k} - L_{m_k} I_0) v \leq 0 \nonumber\\
    \implies& v^\top\left(\sum_{k=1}^2\Psi\left( z_k \right) \nabla^2_{\boldsymbol{\hat{\Theta}_{mm}}}f_{m_k} - \sum_{k=1}^2L_{m_k} I_0 \right) v \leq 0 \nonumber\\
    \implies& \sum_{k=1}^2\Psi \nabla^2_{\boldsymbol{\hat{\Theta}_{mm}}}f_{m_k} \preceq \sum_{k=1}^2 L_{m_k} I_0.
\end{align}
Next, considering the second term of \eqref{eq:g-hessian-2}, we have for any $v\in\R^{\text{dim}(\boldsymbol{\hat{\Theta}_{mm}})}$, 
\begin{align}
    &v^\top \Bigg[ \sum_{k=1}^2\Psi\left( z_k \right)\nabla_{\boldsymbol{\hat{\Theta}_{mm}}}f_{m_k}\nabla_{\boldsymbol{\hat{\Theta}_{mm}}}f_{m_k}^\top - \left(\sum_{k=1}^2\Psi\left( z_k \right) \nabla_{\boldsymbol{\hat{\Theta}_{mm}}}f_{m_k} \right)\left(\sum_{k=1}^2\Psi\left( z_k \right) \nabla_{\boldsymbol{\hat{\Theta}_{mm}}}f_{m_k} \right)^\top \nonumber\\ 
    &~~- \sum_{k=1}^2 L^2_{m_k, 1}I_0\Bigg] v \nonumber\\
    &= \sum_{k=1}^2\Psi\left( z_k \right) y_k^2 - \left(\sum_{k=1}^2\Psi\left( z_k \right)y_k \right)^2 - \Vert v \Vert^2 \sum_{k=1}^2 L^2_{m_k, 1}, ~~\text{$y_k = v^\top\nabla_{\boldsymbol{\hat{\Theta}_{mm}}}f_{m_k}$ for $k\in\{1, 2\}$} \nonumber\\
    &\leq \sum_{k=1}^2\Psi\left( z_k \right) y_k^2 - \Vert v \Vert^2 \sum_{k=1}^2 L^2_{m_k, 1} \nonumber\\
    &\leq \sum_{k=1}^2\Psi\left( z_k \right) \Vert \nabla_{\boldsymbol{\hat{\Theta}_{mm}}}f_{m_k} \Vert^2 \Vert v \Vert^2 - \Vert v \Vert^2 \sum_{k=1}^2 L^2_{m_k, 1}, ~~\text{due to Cauchy-Schwarz inequality}\\
    &\leq \Vert v \Vert^2\sum_{k=1}^2\Psi\left( z_k \right)\left( \Vert \nabla_{\boldsymbol{\hat{\Theta}_{mm}}}f_{m_k} \Vert^2 -L^2_{m_k, 1}\right)\\
    &\leq 0,
\end{align}
where the last inequality is due to Assumption \ref{ass:lip}. The above inequality suggests that
\begin{equation}\label{eq:g-hessian-2-2}
    \sum_{k=1}^2\Psi\left( z_k \right)\nabla_{\boldsymbol{\hat{\Theta}_{mm}}}f_{m_k}\nabla_{\boldsymbol{\hat{\Theta}_{mm}}}f_{m_k}^\top - \left(\sum_{k=1}^2\Psi\left( z_k \right) \nabla_{\boldsymbol{\hat{\Theta}_{mm}}}f_{m_k} \right)\left(\sum_{k=1}^2\Psi\left( z_k \right) \nabla_{\boldsymbol{\hat{\Theta}_{mm}}}f_{m_k} \right)^\top \preceq \sum_{k=1}^2 L^2_{m_k, 1}I_0.
\end{equation}

Putting together \eqref{eq:g-hessian-2}, \eqref{eq:g-hessian-2-1}, and \eqref{eq:g-hessian-2-2}, we have
\begin{equation}\label{eq:g-hessian}
    \nabla^2_{\boldsymbol{\hat{\Theta}_{mm}}}g_\mu \preceq \sum_{k=1}^2 \left(L_{m_k} + \frac{L^2_{m_k, 1}}{\mu} \right) I_0.
\end{equation}
On the other hand, from the $L_{mm}$-smoothness of $f_{mm}$ (Assumption \ref{ass:smooth}), we have 
\begin{equation}\label{eq:fmm-hessian}
    \nabla^2_{\boldsymbol{\hat{\Theta}_{mm}}}f_{mm} \preceq L_{mm} I_0.
\end{equation}
Since $\hat{f}_{mm} = f_{mm} + \lambda g_\mu$, we can have
\begin{equation}\label{eq:fmm-hat-hessian}
    \nabla^2_{\boldsymbol{\hat{\Theta}_{mm}}}\hat{f}_{mm} \preceq \hat{L}_{mm} I_0,
\end{equation}
where $\hat{L}_{mm} := L_{mm} + \lambda \sum_{k=1}^2 \left(L_{m_k} + \frac{L^2_{m_k, 1}}{\mu} \right) > 0 $, which completes the proof.  
\end{proof}

\section{Proof of Theorem \ref{thm:penalty_equ}} \label{app:penalty-equ-proof}

To prove Theorem \ref{thm:penalty_equ}, it is crucial to analyze the properties of the maximum gap function $g(\boldsymbol{\Theta}):=\max_{{\scriptscriptstyle k\in[K]}} \left( f_{m_k}(\boldsymbol{\Theta_{m_k}}) - f^*_{m_k} \right)$, where $\boldsymbol{\Theta_{m}}=[\boldsymbol{\Theta_{m_1}}, \cdots, \boldsymbol{\Theta_{m_K}}]$. As we will show later, although it is non-differentiable, it still has some desired properties. 

\begin{Definition}[Semi-smoothness]
Let $\ell:\mathbb{R}^d\rightarrow\mathbb{R}$ be a locally Lipschitz function. The function $\ell$ is said to be semi-smooth if it is directionally differentiable at any $z\in\mathbb{R}^d$ and for any direction $d\in\mathbb{R}^d$ and any $J\in\partial \ell(z+d)$, it holds that 
\begin{align}
|\ell(z+d)-\ell(z)-J^\top d|={\cal O}(\|d\|^{2}) \text{ as } d\rightarrow 0
\end{align}
Furthermore, the function $\ell$ is said to be semi-smooth on $\mathcal{Z}\subset\mathbb{R}^d$ with $L$ if $\ell$ is  $\alpha$-order semi-smooth at every $z\in\mathcal{Z}$ and there exists $\bar d$ such that for all $\|d\|\leq\bar d$ and any $z\in\mathcal{Z}$, the following holds 
\begin{align}
|\ell(z+d)-\ell(z)-J^\top d|\leq L\|d\|^{2}. 
\end{align}
\end{Definition}
Both smooth and piecewise smooth functions are semi-smooth functions \citep{ulbrich2011semismooth}. 

\begin{Proposition}\label{prop:semismooth-KL}
Under Assumption \ref{ass:PL}--\ref{ass:smooth}, the maximum gap function $g(\boldsymbol{\Theta_{m}}):=\max_{{\scriptscriptstyle k\in[K]}} \left( f_{m_k}(\boldsymbol{\Theta_{m_k}}) - f^*_{m_k} \right)$ satisfies the semi-smooth condition with $L:=\max_{k\in [K]} L_{m_k}$ and the Kurdyka-Łojasiewicz (KL) condition, i.e. \begin{align*}
\operatorname{dist}(0,\partial g(\boldsymbol{\Theta_{m}}))^2 \geq \frac{\mu}{K}g(\boldsymbol{\Theta_{m}})  . 
\end{align*}
\end{Proposition}
\begin{proof}
Since piecewise smooth function is semi-smooth \citep{ulbrich2011semismooth}, $\max_{k\in [K]}:\mathbb{R}^K\rightarrow\mathbb{R}$ is semi-smooth with modulus $1$. Then since the class of semi-smooth functions is closed under composition \citep{ulbrich2011semismooth},  we know that $g(\boldsymbol{\Theta_{m}})$ is semi-smooth over $\boldsymbol{\Theta_{m}}$ with $L:=\max_{k\in [K]} L_{m_k}$. 

Define the active index set as $\mathcal{I}(\boldsymbol{\Theta_{m}})=\arg\max_{k\in [K]}\{ f_{m_k}(\boldsymbol{\Theta_{m_k}}) - f^*_{m_k} \}$ and the block vector $\mathbf{G}^{(k)}=\mathbf{e}_k \otimes \nabla f_{m_k}\left(\boldsymbol{\Theta_{m_k}}\right)$, where $\mathbf{e}_k \in \mathbb{R}^K$ is the $k$-th standard basis vector. This means
$$
\mathbf{G}^{(k)}=\left(0, \ldots, 0, \nabla f_{m_k}\left(\boldsymbol{\Theta_{m_k}}\right), 0, \ldots, 0\right) .
$$
Then the subgradient of $g(\boldsymbol{\Theta_{m}})$ is given by \citep{rockafellar2009variational}
\begin{align}\label{subgrad_set}
\partial g(\boldsymbol{\Theta_{m}})=\operatorname{conv}\{\mathbf{G}^{(k)} ~|~ k\in \mathcal{I}(\boldsymbol{\Theta_{m}})\}. 
\end{align}
For any $v\in\partial g(\boldsymbol{\Theta_{m}})$, there exists $a_k\geq 0$ such that $\sum_{k\in\mathcal{I}(\boldsymbol{\Theta_{m}})} a_k=1$ and $k$-th block of $v$ has the form 
\begin{align}
v_k= \begin{cases}a_k\nabla f_{m_k}\left(\boldsymbol{\Theta_{m}}\right) , & k\in\mathcal{I}(\boldsymbol{\Theta_{m}}), \\ 0 & k\not\in\mathcal{I}(\boldsymbol{\Theta_{m}}) .\end{cases}
\end{align}
Therefore, we have 
\begin{align*}
\|v\|^2&=\sum_{k=1}^K \|v_j\|^2=\sum_{k\in\mathcal{I}(\boldsymbol{\Theta_{m}})} \|a_k\nabla f_{m_k}\left(\boldsymbol{\Theta_{m_k}}\right)\|^2\\
&\stackrel{(a)}{\geq}\mu\sum_{k\in\mathcal{I}(\boldsymbol{\Theta_{m}})} a_k^2 (f_{m_k}\left(\boldsymbol{\Theta_{m_k}}\right)-f_{m_k}^*)\\
&\stackrel{(b)}{=}\mu\sum_{k\in\mathcal{I}(\boldsymbol{\Theta_{m}})} a_k^2 g\left(\boldsymbol{\Theta_{m}}\right)\\
&\stackrel{(c)}{\geq} \frac{1}{|\mathcal{I}(\boldsymbol{\Theta_{m}})|} \mu g\left(\boldsymbol{\Theta_{m}}\right)\geq \frac{\mu}{K}  g\left(\boldsymbol{\Theta_{m}}\right)
\end{align*}
where $(a)$ is because of Assumption \ref{ass:PL}, $(b)$ is due to $k\in\mathcal{I}(\boldsymbol{\Theta_{m}})$, and $(c)$ is because of the Cauchy–Schwarz inequality which gives $\sum_{k\in\mathcal{I}(\boldsymbol{\Theta_{m}})} a_k^2\geq \frac{1}{|\mathcal{I}(\boldsymbol{\Theta_{m}})|}$. Taking inf over both sides for $v\in\partial g(\boldsymbol{\Theta_{m}})$, this completes the proof because $\operatorname{dist}(0,\partial g(\boldsymbol{\Theta_{m}}))^2=\min_{v\in\partial g(\boldsymbol{\Theta_{m}})}\|v\|^2$. 
\end{proof}


\begin{proof}[Proof of Theorem \ref{thm:penalty_equ}]
We only prove the local solution relations, and the corresponding global result then follows by extending the local region to the entire domain. Since $(\boldsymbol{\Theta_{mm}^*}, \boldsymbol{\Theta_{m_k}^*})$ is the $\epsilon_\lambda$-local solution of penalty problem ~\eqref{eq:mimo}, then there exists a local neighborhood of $(\boldsymbol{\Theta_{mm}^*}, \boldsymbol{\Theta_{m_k}^*})\in\mathbb{B}((\boldsymbol{\Theta_{mm}^*}, \boldsymbol{\Theta_{m_k}^*}),R)$ with $R>0$ such that for any $(\boldsymbol{\Theta_{mm}}, \boldsymbol{\Theta_{m_k}})\in\mathbb{B}((\boldsymbol{\Theta_{mm}^*}, \boldsymbol{\Theta_{m_k}^*}),R))$, we have 
\begin{align}\label{penalty_approximate2}
f_{mm}(\boldsymbol{\Theta_{mm}})+\lambda \max_{{\scriptscriptstyle k\in[K]}} \left( f_{m_k}(\boldsymbol{\Theta_{m_k}}) - f^*_{m_k} \right)\geq f_{mm}(\boldsymbol{\Theta_{mm}^*})+\lambda \max_{{\scriptscriptstyle k\in[K]}} \left( f_{m_k}(\boldsymbol{\Theta_{m_k}^*}) - f^*_{m_k} \right)-\epsilon_\lambda 
\end{align}
Specially, for any $(\boldsymbol{\Theta_{mm}}, \boldsymbol{\Theta_{m_k}})\in\mathbb{B}((\boldsymbol{\Theta_{mm}^*}, \boldsymbol{\Theta_{m_k}^*}),R))$ with $$\max_{{\scriptscriptstyle k\in[K]}} \left( f_{m_k}(\boldsymbol{\Theta_{m_k}}) - f^*_{m_k} \right)\leq \max_{{\scriptscriptstyle k\in[K]}} \left( f_{m_k}(\boldsymbol{\Theta_{m_k}^*}) - f^*_{m_k} \right)=:\epsilon$$
we have 
\begin{align}\label{epsilon_optimal2}
f_{mm}(\boldsymbol{\Theta_{mm}})\geq f_{mm}(\boldsymbol{\Theta_{mm}^*})-\epsilon_\lambda 
\end{align}
Therefore, $(\boldsymbol{\Theta_{mm}^*}, \boldsymbol{\Theta_{m_k}^*})$ is $\epsilon_\lambda$-local-optimal to the relaxed problem if $\epsilon:=\max_{{\scriptscriptstyle k\in[K]}} \left( f_{m_k}(\boldsymbol{\Theta_{m_k}^*}) - f^*_{m_k} \right)$ is small.  




Without loose of generality (W.L.O.G), we assume $0\notin\nabla_{\boldsymbol{{\Theta}_{m_k}}} f_{mm}(\boldsymbol{\Theta_{mm}^*})+\lambda \partial_{\boldsymbol{{\Theta}_{m_k}}}\max_{{\scriptscriptstyle k\in[K]}} \left( f_{m_k}(\boldsymbol{\Theta_{m_k}^*}) - f^*_{m_k} \right)$. Denote $\boldsymbol{\hat{\Theta}_{mm}^*}:=[\boldsymbol{{\Theta}_{mm}^*},\{\boldsymbol{{\Theta}_{m_k}^*}\}_{k=1}^K]$, and letting ${ \boldsymbol{\hat{\Theta}_{m_k}^\prime}}=[\boldsymbol{{\Theta}_{mm}^\prime},\{\boldsymbol{{\Theta}_{m_k}^\prime}\}_{k=1}^K]=\boldsymbol{\hat{\Theta}_{m_k}^*}-\frac{ru}{\|u\|}$ where $$u\in \nabla_{\boldsymbol{{\Theta}_{m_k}}} f_{mm}(\boldsymbol{\Theta_{mm}^*})+\lambda \partial_{\boldsymbol{{\Theta}_{m_k}}}\max_{{\scriptscriptstyle k\in[K]}} \left( f_{m_k}(\boldsymbol{\Theta_{m_k}^*}) - f^*_{m_k} \right)$$ 
is some subgradient that will be chosen later and $r\leq R$. Then according to Proposition \ref{prop:semismooth-KL} and Taylor expansion of $f$, we have  
\begin{align*}
-\epsilon_\lambda&\leq f_{mm}(\boldsymbol{\Theta_{mm}^\prime})+\lambda \max_{{\scriptscriptstyle k\in[K]}} \left( f_{m_k}(\boldsymbol{\Theta_{m_k}^\prime}) - f^*_{m_k} \right) -(f_{mm}(\boldsymbol{\Theta_{mm}^*})+\lambda \max_{{\scriptscriptstyle k\in[K]}} \left( f_{m_k}(\boldsymbol{\Theta_{m_k}^*}) - f^*_{m_k} \right))\\
&\leq u^\top (\boldsymbol{\Theta_{m_k}^\prime}-\boldsymbol{\Theta_{m_k}^*})+\frac{L_{mm}+\lambda\max L_{m_k}}{2}\|\boldsymbol{\Theta_{m_k}^\prime}-\boldsymbol{\Theta_{m_k}^*}\|^{2}\leq -r\|u\|+\frac{L_{mm}+\lambda\max L_{m_k}}{2}r^{2}
\end{align*}
Therefore, when choosing $r={\cal O}(R(\epsilon_\lambda/\lambda)^{\frac{1}{2}})\leq R$, we have 
\begin{align*}
&~~~~\operatorname{dist}(0,\nabla_{\boldsymbol{{\Theta}_{m_k}}} f_{mm}(\boldsymbol{\Theta_{mm}^*})+\lambda \partial_{\boldsymbol{{\Theta}_{m_k}}}\max_{{\scriptscriptstyle k\in[K]}} \left( f_{m_k}(\boldsymbol{\Theta_{m_k}^*}) - f^*_{m_k} \right))\\
&\leq\|u\|\leq\frac{L_{mm}+\lambda\max L_{m_k}}{2}r+\frac{\epsilon_\lambda}{r}={\cal O}(\lambda r +\epsilon_\lambda/r)={\cal O}(\epsilon_\lambda^{\frac{1}{2}}\lambda^{\frac{1}{2}})
\end{align*}
Let $u_g\in\partial_{\boldsymbol{{\Theta}_{m_k}}}\max_{{\scriptscriptstyle k\in[K]}} \left( f_{m_k}(\boldsymbol{\Theta_{m_k}^*}) - f^*_{m_k} \right)$ be the one with $u=\nabla_{\boldsymbol{{\Theta}_{m_k}}} f_{mm}(\boldsymbol{\Theta_{mm}^*})+\lambda u_g$, then by Cauchy-Schwartz inequality, 
\begin{align*}
\operatorname{dist}(0,\partial_{\boldsymbol{{\Theta}_{m_k}}}\max_{{\scriptscriptstyle k\in[K]}}\left( f_{m_k}(\boldsymbol{\Theta_{m_k}^*}) - f^*_{m_k} \right))\leq\|u_g\|\leq\frac{\|u\|}{\lambda}+\frac{\|\nabla_{\boldsymbol{{\Theta}_{m_k}}} f_{mm}(\boldsymbol{\Theta_{mm}^*})\|}{\lambda}\leq {\cal O}\left(\sqrt{\frac{\epsilon_\lambda}{\lambda}}+\frac{1}{\lambda}\right)={\cal O}(\epsilon_\lambda). 
\end{align*}
where the last equality is obtained by setting $\lambda={\cal O}(\epsilon_\lambda^{-1})$. 
Furthermore, with Proposition \ref{prop:semismooth-KL}, we have 
\begin{align*}
\frac{\mu}{K}\left(\max_{{\scriptscriptstyle k\in[K]}}\left( f_{m_k}(\boldsymbol{\Theta_{m_k}^*}) - f^*_{m_k} \right))\right)\leq{\cal O}(\epsilon_\lambda^2)
\end{align*}
Therefore, we have $\epsilon:=\max_{{\scriptscriptstyle k\in[K]}}\left( f_{m_k}(\boldsymbol{\Theta_{m_k}^*}) - f^*_{m_k} \right))\leq{\cal O}(\epsilon_\lambda^2)$. Combining with \eqref{epsilon_optimal2},  $(\boldsymbol{\Theta_{mm}^*}, \boldsymbol{\Theta_{m_k}^*})$ is $\epsilon_\lambda$-local-optimal to the relaxed problem with $\epsilon$.  
\end{proof}

\section{Additional Experiments and Details}\label{app:experiments}

In this section, we provide implementation details for MIMO and other baselines in CREMA-D, AV-MNIST, UR-FUNNY, and CMU-MOSEI multi-modal sensor benchmark datasets. Furthermore, we provide an ablation of MIMO parameters. For implementing uni-modal sensor learning, vanilla MSL, and balanced multi-modal sensor learning methods, we use the implementations of \citep{li2023boosting}\footnote{\url{https://github.com/lihong2303/AGM_ICCV2023.git}} and \citep{peng2022balanced}\footnote{\url{https://github.com/GeWu-Lab/OGM-GE_CVPR2022}}.  MOO baselines are implemented by us. To be comparable with MOO methods, for uni-modal sensor accuracy results for vanilla-MSL and balanced multi-modal sensor learning methods, we train a dedicated uni-modal sensor head using the features extracted from the uni-modal sensor encoders, in addition to the fusion heads.  We provide an average of over three seeds for our experiments, with an error bar of one standard deviation. All experiments are run using 2 NVIDIA GeForce RTX 3090 GPUs and 4 NVIDIA RTX A6000 GPUs.

\paragraph{CREMA-D \citep{cao2014crema}.} This dataset is for multi-modal speech emotion recognition using facial and vocal expressions. The dataset includes six common emotions: anger, happiness, sadness, neutrality, disgust, and fear. It is randomly divided into a training set with $6,027$ samples, a validation set with $669$ samples, and a testing set with $745$ samples. For method-specific parameter configurations for implementing uni-modal sensor learning, vanilla MSL, and balanced multi-modal sensor learning baselines we use the default setting of implementation by \citep{li2023boosting}. All methods are optimized with SGD optimizer with an initial stepsize of $10^{-3}$, for $100$ epochs.

\paragraph{UR-Funny \citep{hasan2019ur}.}The dataset was created for affective computing tasks that detect humor through the use of words (text), gestures (vision), and prosodic cues (acoustic). This dataset was collected from TED talks and utilizes an equal number of binary labels for each sample. For method-specific parameter configurations for implementing uni-modal sensor learning, vanilla MSL, and balanced multi-modal sensor learning baselines we use the default setting of implementation by \citep{li2023boosting}. All methods are optimized with SGD optimizer with an initial stepsize of $10^{-3}$, for $100$ epochs. Note that OGM-GE method is not implemented in this dataset since OGM-GE is by design only a two-modality balanced multi-modal sensor learning method.

\noindent \textbf{Kinetics-Sounds \citep{arandjelovic2017look}} This dataset is derived from the larger Kinetics dataset~\cite{kay2017kinetics}, which contains 400 classes of YouTube videos. Kinetics-Sounds specifically includes 31 human action categories that were selected for their potential to be both seen and heard, such as playing musical instruments. Each video clip is 10 seconds long, manually labeled for human actions via Mechanical Turk, and trimmed to center around the action of interest. The dataset comprises 19,000 video clips in total, with a split of 15,000 for training, 1,900 for validation, and 1,900 for testing.

\noindent \textbf{VGGSound \citep{chen2020vggsound}} This dataset is a comprehensive video dataset consisting of 309 classes that span a broad spectrum of audio events encountered in everyday contexts. The videos, each lasting 10 seconds, are recorded in real-world settings with an audio-visual alignment, meaning the source of the sound is visible. The dataset is partitioned following the original split in~\cite{chen2020vggsound}. For our experiments, 168,618 videos are used for training and validation, while 13,954 are allocated for testing due to the unavailability of some YouTube videos.

\begin{table}[t]
\begin{center}
\caption{Comparison using AV-MNIST dataset.}\label{tab:avminst}
\begin{tabular}{l|cccc}
\hline
Method & Acc {\tiny(\%)} & Acc$_a$ {\tiny(\%)} & Acc$_v$ {\tiny(\%)} & $t(s)$ \\
\hline
$\mathbf{\mathcal{C}^a}$    &  - & \textbf{42.32}{\tiny $\pm$ 0.17} & - & 0.019{\tiny $\pm$ 0.003}\\
$\mathbf{\mathcal{C}^v}$    &  - & - & \textbf{65.05}{\tiny $\pm$ 0.08}  & 0.019{\tiny $\pm$ 0.003}  \\
\textbf{Vanilla MSL}  &  71.70{\tiny $\pm$ 0.11} & 39.98{\tiny $\pm$ 0.46} & 64.67{\tiny $\pm$ 0.24}  & 0.015{\tiny $\pm$ 0.002}  \\
\hline
\textbf{MSES}   &  71.61{\tiny $\pm$ 0.02} & 39.92{\tiny $\pm$ 0.65} & 64.66{\tiny $\pm$ 0.26}  & 0.017{\tiny $\pm$ 0.002}  \\
\textbf{MSLR}   &  71.96{\tiny $\pm$ 0.12} & 40.5{\tiny $\pm$ 0.79} & 64.50{\tiny $\pm$ 0.14}  & 0.018{\tiny $\pm$ 0.002}  \\
\textbf{OGM-GE}  &  71.70{\tiny $\pm$ 0.11} & 39.98{\tiny $\pm$ 0.46} & 64.67{\tiny $\pm$ 0.24}  & 0.055{\tiny $\pm$ 0.026}  \\
\textbf{AGM}   &  70.92{\tiny $\pm$ 0.81} & 29.16{\tiny $\pm$ 0.92} & 61.76{\tiny $\pm$ 1.98}  & 0.060{\tiny $\pm$ 0.021}  \\
\hline
\textbf{EW}     &  72.22{\tiny $\pm$ 0.04} & 41.63{\tiny $\pm$ 0.26} & 41.77{\tiny $\pm$ 0.18}  & 0.019{\tiny $\pm$ 0.011} \\
\textbf{MGDA}     &  72.15{\tiny $\pm$ 0.50} & 41.63{\tiny $\pm$ 0.12} & 41.92{\tiny $\pm$ 0.22}  & 0.086{\tiny $\pm$ 0.003}  \\
\textbf{MMPareto}     &  72.42{\tiny $\pm$ 0.21} & 41.64{\tiny $\pm$ 0.35} & 41.83{\tiny $\pm$ 0.25}  & 0.085{\tiny $\pm$ 0.002}  \\
\textbf{MIMO (ours)}     &  \textbf{72.77}{\tiny $\pm$ 0.10} & 42.21{\tiny $\pm$ 0.38} & 42.25{\tiny $\pm$ 0.34}  &  \underline{0.018}{\tiny $\pm$ 0.004} \\
\bottomrule
\end{tabular}
\end{center}
\end{table}

\paragraph{AV-MNIST \citep{vielzeuf2018centralnet}.}In addition to the experiments given in the main text, here we provide a comparison of MIMO with proposed baselines in the AV-MNIST dataset. This dataset is for multi-media classification tasks by combining visual and audio features. The first modality, a noisy image, consists of $28 \times 28$ PCA-projected MNIST images. The second modality, audio, consists of audio samples represented by $112 \times 122$ spectrograms. The entire dataset comprises $70,000$ samples, divided into a training set and a validation set at a ratio of $6:1$. Additionally, $10\%$ of the samples from both the training set and validation set were randomly selected to create a development set. For method-specific parameter configurations for implementing uni-modal sensor learning, vanilla MSL, and balanced multi-modal sensor learning baselines we use the default setting of implementation by \citep{li2023boosting}. All methods are optimized with SGD optimizer with an initial stepsize of $10^{-3}$, for $100$ epochs. From Table \ref{tab:avminst},  it can be seen that MIMO can outperform the best performing modality significantly, and perform comparably or better compared to other baselines. Moreover, when considering the subroutine execution times, MIMO is $\sim 4$ faster compared to the next best performing method (MMPareto). These results demonstrate that MIMO can achieve superior performance with balanced multi-modal sensor learning, incurring only a minimal increase in computational time.

\paragraph{CMU-MOSEI \citep{zadeh2018multimodal}.} This dataset was compiled for sentence-level sentiment analysis and emotion recognition, consisting of 23,454 movie review clips drawn from over 65.9 hours of YouTube video featuring 1,000 speakers. As per the implementation in \citep{li2023boosting}, we utilize only the text and audio modalities, and the train/validation/test sets are split into 16,327, 1,871, and 4,662 samples, respectively. All methods are optimized with SGD optimizer with an initial stepsize of $10^{-4}$, for $100$ epochs. The experiment results for CMU-MOSEI dataset is given in Table \ref{tab:cmu-mosei}. It can be seen that while MIMO can outperform the best performing modality, vanilla MSL and AGM fail to achieve this. Moreover, when considering the subroutine execution times, MIMO is only slightly slower than vanilla MSL. These results demonstrate that MIMO can achieve superior performance with balanced multi-modal sensor learning, incurring only a minimal increase in computational time.

\paragraph{AVE \citep{tian2018audio}.} This dataset is an audio-visual dataset designed for event localization, encompassing 28 event classes. It contains 4,143 videos, each 10 seconds long, with synchronized audio and visual tracks, along with frame-level annotations. The experiment results for AVE dataset is given in Table \ref{tab:ave}.  It can be seen that, similar to CMU-MOSEI dataset, MIMO can achieve superior performance compared to the baselines with a subroutine time similar to that of vanilla MSL. Furthermore, it can be seen that MIMO can outperform unimodal baselines consistently, while vanilla MSL and AGM fail to achieve this.

\noindent
\begin{minipage}[t]{0.48\linewidth}
\centering
\setlength{\tabcolsep}{0.1em}
\captionof{table}{Comparison using CMU-MOSEI dataset.}\label{tab:cmu-mosei}
{\small
\begin{tabular}{lcccc}
\toprule
Method & Acc {\tiny(\%)} & Acc$_t$ {\tiny(\%)} & Acc$_a$ {\tiny(\%)} & $t(s)$ \\
\midrule
Text    &  - & \textbf{81.53}{\tiny $\pm$ 0.16} & - & 0.100{\tiny $\pm$ 0.009}\\
Audio    &  - & - & 74.12{\tiny $\pm$ 0.06} & 0.101{\tiny $\pm$ 0.009}  \\
\hline
\textbf{MSL}     &  80.33{\tiny $\pm$ 0.18} & 73.89{\tiny $\pm$1.58} & 73.08{\tiny $\pm$ 0.01} & \underline{0.279}{\tiny $\pm$ 0.009} \\
\textbf{AGM}   &  80.28{\tiny $\pm$ 0.19} & 79.64{\tiny $\pm$0.19} & 78.23{\tiny $\pm$ 0.65} & 0.304{\tiny $\pm$ 0.007}  \\
\textbf{MIMO}    &  \textbf{81.62}{\tiny $\pm$ 0.06} & 81.44{\tiny $\pm$ 0.10} & \textbf{81.36} {\tiny $\pm $0.23}  & 0.287{\tiny $\pm$ 0.009} \\
\bottomrule
\end{tabular}}
\end{minipage}%
\begin{minipage}[t]{0.48\linewidth}
\centering
\setlength{\tabcolsep}{0.1em}
\captionof{table}{Comparison using AVE dataset.}\label{tab:ave}
{\small
\begin{tabular}{lcccc}
\toprule
Method & Acc {\tiny(\%)} & Acc$_a$ {\tiny(\%)} & Acc$_v$ {\tiny(\%)} & $t(s)$ \\
\midrule
Audio    &  - & $66.03${\tiny $\pm 0.28$} & - & $0.010${\tiny $\pm 0.001$}\\
Visual    &  - & - & $63.82${\tiny $\pm 0.99$} & $0.011${\tiny $\pm 0.001$} \\
\hline
\textbf{MSL}     &  $67.41${\tiny $\pm 0.30$} & $33.46${\tiny $\pm 1.49$} & $56.61${\tiny $\pm 1.91$} & $\underline{0.027}${\tiny $\pm 0.002$}\\
\textbf{AGM}   &  $72.54${\tiny $\pm 1.13$} & $54.73${\tiny $\pm 1.29$} & $50.92${\tiny $\pm 1.98$} & $0.161${\tiny $\pm 0.002$}\\
\textbf{MIMO}    &  $\bm{73.69}${\tiny $\pm 0.24$} & $\bm{72.69}${\tiny $\pm 0.22$} & $\bm{71.85}$ {\tiny $\pm 0.47$}  & $0.029${\tiny $\pm 0.002$}\\
\bottomrule
\end{tabular}}
\end{minipage}

\paragraph{MOO baseline implementation.} We implement the equal weighting (EW) method by optimizing the sum of uni-modal sensor and multi-modal sensor fusion objectives. For implementing MGDA, we consider the shared and non-shared parameters separately. Specifically, we solve the MGDA sub-problem \citep{fliege2019complexity} using the gradient of uni-modal sensor and multi-modal sensor fusion objectives with respect to encoder weights for each modality encoder. Non-shared parameters like multi-modal and uni-modal sensor heads are updated using normal SGD updates. For MMPareto, we follow the method described in \citep{wei2024innocent}, with updating shared and non-shared parameters similar to that of MGDA implementation.

\paragraph{Subroutine time calculation.} For calculating the subroutine times of MIMO and baselines, we compute the average computation time taken for the subroutine used for balancing modalities (if any) and updating the model parameters per batch. Since run times differ for different seeds due to background processes, we report the average subroutine times (over $100\times$ number of batches per epoch) calculated using one seed.

\begin{table}
    \centering
    \caption{MIMO parameters}
    \begin{tabular}{|l|c|c|c|c|c|c|c|}
    \hline
         & CREMA-D  & UR-Funny & Kinetics-Sound & VGGSound & AV-MNIST & CMU-MOSEI & AVE\\
         \hline\hline
       $\lambda$  & 100  & 10 & 10 & 10 & 10  & 100 & 100\\
       $\mu$      & 0.01 & 1.0 & 0.01 & 0.01 & 0.1 & 0.001 & 0.01\\
       \hline
    \end{tabular}
    \label{tab:mimo-param}
\end{table}

\paragraph{MIMO parameters and implementation.} MIMO-specific parameters used for each dataset are given in Table \ref{tab:mimo-param}. We coarsely tune $\lambda$ parameter in the grid $\{1, 10, 100\}$, and $\mu$ parameter in the grid $\{0.001, 0.01, 0.1, 1.0\}$ for each dataset. To ensure numerical stability during MIMO implementation, when the loss values become large, we increase the value of $\mu$ two times until the exponents in the MIMO objective fall within the permissible range for the datatype. The reported subroutine times include the computation time required for this adjustment.

\end{document}